\documentclass[dvipsnames,sn-mathphys,Numbered]{sn-jnl}

\usepackage{amsmath,amssymb,amsfonts}
\usepackage{graphicx}
\usepackage{textcomp}
\usepackage{pdfsync}
\usepackage{lscape}
\usepackage{rotating}
\usepackage{xcolor}
\usepackage{manyfoot}\usepackage{csvsimple}
\usepackage{multirow}
\usepackage{pgfplots}

\usepackage{amsthm}
\usepackage{listings}
\usepackage{xspace}

\usepackage{url}

\usepackage{subcaption}

\lstdefinelanguage{Julia}{morekeywords={abstract,break,case,catch,const,continue,do,else,elseif,end,export,false,for,function,immutable,import,importall,if,in,macro,module,otherwise,quote,return,switch,true,try,type,typealias,using,while},sensitive=true,morecomment=[l]\#,morecomment=[n]{\#=}{=\#},morestring=[s]{"}{"},morestring=[m]{'}{'},}[keywords,comments,strings]

\definecolor{cadmiumgreen}{rgb}{0.0, 0.42, 0.24}

\newtheorem{theorem}{Theorem}\newtheorem{definition}{Definition}\newtheorem{remark}{Remark}

\newcommand{\mat}[1]{\mathbf{#1}}
\renewcommand{\vec}[1]{{\boldsymbol{#1}}}
\newcommand{\T}[1]{\mat{#1}^\top}
\newcommand{\avgbeta}{\vec{\bar{\beta}}}
\newcommand{\normalpha}{\hat{\vec{\alpha}}}
\newcommand{\avgbetaat}[1]{\bar{\beta}_{#1}}
\newcommand{\normbeta}{\hat{\vec{\beta}}}
\newcommand{\PLS}{Par\-ti\-tio\-ned\-LS\xspace}
\newcommand{\algoalt}{PartLS-alt\xspace}
\newcommand{\algoopt}{PartLS-opt\xspace}
\newcommand{\algobnb}{PartLS-bnb\xspace}

\newcommand{\subsetsum}{\textsc{subset sum}\xspace}
\newcommand{\ub}{\ensuremath{\mu}\xspace}

\begin{document}

\title[Partitioned Least Squares]{Partitioned Least Squares}

\author*[1]{\fnm{Roberto} \sur{Esposito}}\email{roberto.esposito@unito.it}

\author[2]{\fnm{Mattia} \sur{Cerrato}}\email{mcerrato@uni-mainz.de}
\equalcont{These authors contributed equally to this work.}

\author[3]{\fnm{Marco} \sur{Locatelli}}\email{marco.locatelli@unipr.it}
\equalcont{These authors contributed equally to this work.}

\affil*[1]{\orgdiv{Computer Science Department}, \orgname{University of Turin}, \orgaddress{\street{Corso Svizzera 185}, \city{Turin}, \postcode{10125},  \country{Italy}}}

\affil[2]{\orgdiv{Computer Science Department}, \orgname{Johannes Gutenberg Universit\"at Mainz}, \orgaddress{\street{Staudingerweg 9}, \city{Mainz}, \postcode{55112}, \state{}, \country{Germany}}}

\affil[3]{\orgdiv{Engineering and Architecture Department}, \orgname{University of Parma}, \orgaddress{\street{Parco Area delle Scienze 181/A}, \city{Parma}, \postcode{610101}, \state{43124}, \country{Italy}}}

\abstract{
    Linear least squares is one of the most widely used regression methods in many fields. The simplicity of the model allows this method to be used when data is scarce and allows practitioners to gather some insight into the problem by inspecting the values of the learnt parameters. In this paper we propose a variant of the linear least squares model allowing practitioners to partition the input features into groups of variables that they require to contribute similarly to the final result. We show that the new formulation is not convex and provide two alternative methods to deal with the problem: one non-exact method based on an alternating least squares approach; and one exact method based on a reformulation of the problem. We show the correctness of the exact method and compare the two solutions showing that the exact solution provides better results in a fraction of the time required by the alternating least squares solution (when the number of partitions is small). We also provide a branch and bound algorithm that can be used in place of the exact method when the number of partitions is too large as well as a proof of NP-completeness of the optimization problem.
}
  
\keywords{
  Least Squares, Interpretability, Convex Optimization, NP Completeness
}

\maketitle

\section{Introduction}
\label{sec:introduction}
Linear regression models are among the most extensively employed statistical methods in science and industry alike~\citep{bro2002maximum,intriligator1978econometric,isobe1990linear,nievergelt2000tutorial,reeder2004multicoil}. Their simplicity, ease of use and performance in low-data regimes enables their usage in various prediction tasks. As the number of observations usually exceeds the number of variables, a practitioner has to resort to approximating the solution of an overdetermined system. Least squares approximation benefits from a closed-form solution and is perhaps the most well known approach in linear regression analysis. Among the benefits of linear regression models there is the possibility of easily interpreting how much each variate is contributing to the approximation of the dependent variable by means of observing the magnitudes and signs of the associated parameters. 

In some application domains, partitioning the variables in non-overlapping subsets is beneficial either as a way to insert human knowledge into the regression analysis task or to further improve model interpretability. When considering high-dimensionality data, grouping variables together is also a natural way to make it easier to reason about the data and the regression result. As an example, consider a regression task where the dependent variable is the score achieved by students in an University or College exam. A natural way to group the dependent variables is to divide them into two groups where one contains the variables which represent a student's effort in the specific exam (hours spent studying, number of lectures attended...), while another contains the variables related to previous effort and background (number of previous exams passed, number of years spent at University or College, grade average...). Assuming all these variables could be measured accurately, it might be interesting to know how much each group of variables contributes to the student's score. 
As a further example, when analyzing complex chemical compounds, it is possible to group together fine-grained features to obtain a partition which refers to high-level properties of the compound (such as structural, interactive and bond-forming among others), and knowing how much each high-level property contributes to the result of the analysis is often of great practical value~\citep{caron13block}. The LIMPET dataset that we introduce in Section~\ref{sec:experiments} is a clear-cut example of problems with such structure. In the LIMPET dataset, we have a large number of features that can be grouped in well understood high-level structures and where variables in each group necessarily have to contribute in the same direction (i.e., positively or negatively) to the prediction of lipophilicity of the compound under study.

In this paper, we present a novel variation of linear regression that incorporates feature partitioning into discernible groups. This adapted formulation empowers the analyst to exclude unwanted, unrealistic solutions wherein features within a group are assigned parameters of contrary signs. Thus, the analyst is able to inject domain-specific knowledge into the model. Furthermore, the parameters obtained by solving the problem allow one to easily assess the contribution of each group to the dependent variable as well as the importance of each element of the group. 

The newly introduced problem is not easy to solve and indeed we will prove the non-convexity of the objective, and the NP-completeness of the problem itself. In Section~\ref{sec:algorithms} we introduce two possible algorithms to solve the problem. One is based on an Alternate Convex Search method~\citep{wendell1976minimization}, where the optimization of the parameters is iterative and can get trapped into local minima; the other is based on a reformulation of the original problem into an exponential number of sub-problems, where the exponent is the cardinality $K$ of the partition. We prove convergence of the alternating least square algorithm and the global optimality of the result returned by the second approach. We also provide guidance for building a branch and bound~\citep{lawler1966branch} solution that might be useful when the cardinality of the partition is too large to use the exact algorithm. 

We test the two algorithms on several datasets. Our experiments include data extracted from the analysis of chemical compounds~\citep{caron13block} in a particular setting where this kind of analysis already proved to be of value to practitioners, and a number of
datasets having a large amount of features which we selected from the UCI repository~\citep{dua2019uci}: in this latter case the number, size, and composition of the partition has been decided arbitrarily just to experiment with the provided algorithms. Our experimental results show that the exact algorithm is usually a good choice, the non-exact algorithm being preferable when high accuracy is not required and/or the cardinality of the partition is too large. 
Finally, we present and discuss the application of our algorithms to the problem of predicting house prices, showing that the solution provided by our approach leads to more interpretable and actionable results with respect to a least squares model.

While to the best of our knowledge the regression problem and the algorithms we present are novel, there has been previous work dealing with alternative formulations to the linear regression problem. Some of them have shown to be of great practical use and have received attention from both researchers and practitioners. 

Partial Least Squares (PLS) Regression \citep{wold2009pls} is a very popular method in hard sciences such as chemistry and chemometrics. PLS has been designed to address the undesirable behavior of ordinary least squares when the dataset is small, especially when the number of features is large in comparison. In such cases, one can try to select a smaller set of features allowing a better behavior. A very popular way to select important features is to use Principal Component Analysis (PCA) to select the features that contributes most to the variation in the dataset. However, since PCA is based on the data matrix alone, one risks to filter out features that are highly correlated with the target variables in $\mat{Y}$. PLS has been explicitly designed to solve this problem by decomposing $\mat{X}$ and $\mat{Y}$ simultaneously and in such a way to explain as much as possible of the covariance between $\mat{X}$ and $\mat{Y}$~\citep{abdi10partial}. Our work is substantially different from these approaches since we are not concerned at all with the goal of removing variables. On the contrary, we group them so to inject domain knowledge in the model, make the result more interpretable, and to provide valuable information about the importance of each group.

Yet another set of techniques that resembles our work are those where a partition of the variables is used to \emph{select} groups of features. 
Very well known members of this family of algorithms are group lasso methods~\citep{bakin99adaptive,yuan2006model} (\cite{huang2012selective} provide a review of such methodologies). In these works, the authors tackle the problem of selecting grouped variables for accurate prediction. In this case, as in ours, the groups for the variables are defined by the user, but in their case the algorithm needs to predict which subset of the groups will lead to better performances (i.e., either all variables in a group will be used as part of the solution or none of them will be).
This is a rather different problem with respect to the one that we introduce here.
In our case, we shall assume that all groups are relevant to the analysis. However, in our case we seek a solution where all variables in the same group contributes in the same direction (i.e., with the same sign) to the solution. We argue that this formulation allows for an easier interpretation of the contribution of the whole group as well as of the variables included in each group.

Other techniques that bear some resemblance to our proposal are latent class models~\cite{mccutcheon1987latent}.
Latent class models are a categorical extension to factor analysis, trying to relate a set of observed variables to a set of latent variables. 
The value taken by these latter (usually discrete~\cite{mccutcheon1987latent}) variables should explain much of the variance in the former ones.
Our problem formulation, on the other hand, constrains the prediction of a continuous target variable by grouping together sets of observed variables.
While the solutions found by the algorithms we propose in this paper may reveal interesting patterns (see Section~\ref{sec:experiments}), our method is not unsupervised and cannot be straightforwardly used to describe the variation of the data via discrete, unobservable factors.
Our proposal assumes the availability of a dependent, continuous variable which an analyst is interested in predicting.

In this paper we introduce a new least squares problem and provide algorithms to solve it. We note that we presented the original problem formulation for \PLS in a 2019 paper \cite{esposito2019partitioned}. In this follow-up paper, we provide the following new results:
\begin{itemize}
\item a revised definition for the PartitionedLS-b problem (see Section~\ref{sec:algorithms}), which allows for an improved optimality proof; 
    \item a complete proof of optimality for the optimal algorithm \algoopt, only sketched in previous work \cite{esposito2019partitioned};
    \item a proof of NP-completeness for the PartitionedLS problem (not present in previous work);
    \item a new branch-and-bound algorithm that may be used in conjunction with \algoopt when the number of partitions is high;
    \item information about how to update the algorithms to regularize the solutions;
    \item information about how to leverage the non-negative least squares algorithm~\citep{lawson1995solving} to improve numerical stability;
    \item an experimentation of the optimal and the approximate algorithms over three new datasets;
    \item an experiment showing how the branch-and-bound algorithm compares with the enumerative one; 
    \item a new experiment and a discussion of the interpretability of the results obtained by our approach when applied to the problem of predicting house prices;
    \item a comparison of the generalization performances of our method with Least Squares, Partial Least Squares, and Principal Component Regression.
\end{itemize}

 \section{Model description}
\label{sec:model}

\begin{table}[tb]
  \caption{Notation}
  \label{tab:notation}
  \begin{tabular}{lp{0.8\linewidth}}
    \toprule
    Symbol(s) & Definition \\
    \midrule
    $a_i$ & $i$-th component of vector $\vec{a}$. \\[0.5em]
    $(\cdot)_n$ & Shorthand to specify vectors (or matrices) in terms of their components. For instance $(i)_i$ shall denote a vector $\vec{v}$ such that $v_i = i$. \\[0.5em]
    $k$, $K$ & $k$ is the index for iterating over the $K$ subsets belonging to the partition. \\[0.5em]
    $m$, $M$ & $m$ is the index for iterating over the $M$ variables. \\[0.5em]
    $\mat{X}$ & an $N \times M$ matrix containing the descriptions of the training instances. \\[0.5em]
    $\mat{A} \times \mat{B}$ & matrix multiplication operation (we also simply write it $\mat{A}\mat{B}$ when the notation appears clearer).\\[0.5em]
    $\vec{y}$ & a vector of length $N$ containing the labels assigned to the examples in $\mat{X}$. \\[0.5em]
    $\bullet$ & wildcard used in subscriptions to denote whole columns or whole rows: e.g., $\vec{X}_{\bullet,k}$ denotes the $k$-th column of matrix $\mat{X}$ and $\vec{X}_{m,\bullet}$ denotes its $m$-th row. \\[0.5em]
    $\star$ & denotes an optimal solution, e.g., $p^\star$ denotes the optimal solution of the \PLS problem, while $p_b^\star$ denotes the optimal solution of the \PLS-b problem. \\[0.5em]
    $\mat{P}$ & a $M \times K$ partition matrix, $P_{m,k} \in \{0,1\}$, with $P_{m,k} = 1$  iff variable $\alpha_m$ belongs to the $k$-th element of the partition. \\[0.5em]
    $P_k$ & the set of all indices in the $k$-th element of the partition: $\{ m | P_{k,m} = 1\}$.\\[0.5em]
    $k[m]$ & index of the partition element to which $\alpha_m$ belongs, i.e.: $k[m]$ is such that $m \in P_{k[m]}$.\\[0.5em]
    $\circ$ & Hadamard (i.e., element-wise) product. When used to multiply a matrix by a column vector, it is intended that the columns of the matrix are each one multiplied (element-wise) by the column vector. \\[0.5em]
    $\oslash$ & Hadamard (i.e., element-wise) division. \\[0.5em]
    $\succeq$ & element-wise larger-than operator: $\vec{\alpha} \succeq 0$ is equivalent to $\alpha_m \geq 0$ for $m \in {1..M}$.\\
    \botrule
  \end{tabular}
\end{table}

In this work we denote matrices with capital bold letters such as $\mat{X}$ and vectors with lowercase bold letters as $\vec{v}$. In the text we use a regular (non-bold) font weight when we refer to the \emph{name} of the vector or when we refer to scalar values contained in the vector. In other words, we use the bold face only when we refer to the vector itself. For instance, we might say that the values in the $\alpha$ vector are those contained in the vector $\vec{\alpha}$, which contains in position $i$ the scalar $\alpha_i$. 
We consistently define each piece of notation as soon as we use it, but we also report it in Table \ref{tab:notation}, where the reader can more easily access  the whole notation employed throughout the paper.

Let us consider the problem of inferring a linear least squares model to predict a real variable $y$ given a vector $\vec{x} \in \mathbb{R}^M$. We will assume that the examples are available at learning time as an $N \times M$ matrix $\mat{X}$ and $N \times 1$ column vector $\vec{y}$. We will also assume that the problem is expressed in homogeneous coordinates, i.e., that $\mat{X}$ has an additional column containing values equal to $1$, and that the intercept term of the affine function is included into the weight vector $\vec{w}$ to be computed.

The standard least squares formulation for the problem at hand is to minimize the quadratic loss over the residuals, i.e.:

\begin{gather*}
\text{minimize}_{\vec{w}} \| \mat{X}\vec{w} - \vec{y} \|_2^2.
\end{gather*}
This is a problem that has the closed form solution $\vec{w} = (\T{X} \mat{X})^{-1} \T{X} \vec{y}$. 
As mentioned in Section~\ref{sec:introduction}, in many application contexts where $M$ is large, the resulting model is hard to interpret. 
However, it is often the case that domain experts can partition the elements in the weights vector into a small number of groups and that a model built on this partition would provide more accurate results (by incorporating domain knowledge) or/and be much easier to interpret. 
Then, let $\mat{P}$ be a ``partition'' matrix for the problem at hand (this is not a partition matrix in the linear algebra sense, it is simply a matrix containing the information needed to partition the features of the problem). More formally, let $\mat{P}$ be a $M \times K$ matrix where $P_{m,k} \in \{0,1\}$ is equal to $1$ iff feature number $m$ belongs to the  $k$-th partition element. We will also write $P_k$ to denote the set $\{m | P_{m,k} = 1\}$ of all the features belonging to the $k$-th partition element.

Here we introduce the Partitioned Least Squares (\PLS) problem, a model where we introduce $K$ additional variables and express the whole regression problem in terms of these new variables (and in terms of how the original variables contribute to the predictions made using them). The simplest way to describe the new model is to consider its regression function (to make the discussion easier, we start with the data matrix $\mat{X}$ expressed in non-homogenous coordinates and switch to homogenous coordinates afterwards):
\begin{equation}
  \label{eq:f-scalar-form}
  f(\mat{X}) = \left(\sum_{k=1}^K \beta_k \sum_{m \in P_k} \alpha_m x_{n,m} + t\right)_n ,
\end{equation}
i.e., $f({\bf X})$ computes a vector whose $n$-th component is the one reported within parenthesis (see Table~\ref{tab:notation} for details on the notation).
The first summation is over the $K$ sets in the partition that domain experts have identified as relevant, while the second one iterates over all variables in that set. We note that the $m$-th $\alpha$ weight contributes to the $k$-th element of the partition only if feature number $m$ belongs to it. As we shall see, we require that all $\alpha$ values are nonnegative, and that $\forall k: \sum_{m \in P_k} \alpha_m = 1$. Consequently, the expression returns a vector of predictions calculated in terms of two sets of weights: the $\beta$ weights, which are meant to capture the magnitude and the sign of the contribution of the $k$-th element of the partition, and the $\alpha$ weights, which are meant to capture how each feature in the $k$-th set contributes to it. We note that the $\alpha$ weight vector is of the same length as the vector $\vec{w}$ in the least squares formulation. Despite this similarity, we prefer to use a different symbol because the interpretation of (and the constraints on) the $\alpha$ weights are different with respect to the $w$ weights. 

It is easy to verify that the definition of $f$ in (\ref{eq:f-scalar-form}) can be rewritten in matrix notation as:
\begin{align}
  \label{eq:f-matrix-form-non-homogeneous}
  f(\mat{X})  
  & = \left(\sum_{k=1}^K \beta_k \sum_m P_{m,k} \alpha_m  x_{n,m} + t\right)_n \nonumber \\
  & = \mat{X} \times (\mat{P} \circ \vec{\alpha}) \times \vec{\beta} + t, 
\end{align}

where $\circ$ is the Hadamard product extended to handle column-wise products. More formally, if $\mat{Z}$ is a $A \times B$ matrix, $\vec{1}$ is a $B$ dimensional vector with all entries equal to $1$, and $\vec{a}$ is a  column vector of length $A$, then $\mat{Z} \circ \vec{a} \triangleq \mat{Z} \circ (\vec{a} \times \T{1})$; where the $\circ$ symbol on the right hand side of the definition is the standard Hadamard product.
Equation~(\ref{eq:f-matrix-form-non-homogeneous}) can be rewritten in homogeneous coordinates as:
\begin{equation}
  \label{eq:f-matrix-form}
  f(\mat{X}) = \mat{X} \times (\mat{P} \circ \vec{\alpha}) \times \vec{\beta},
\end{equation}
where $\mat{X}$ incorporates a column with all entries equal to 1, and we consider an additional group (with index $K+1$) having a single $\alpha_{M+1}$ variable in it. Given the constraints on $\alpha$ variables, $\alpha_{M+1}$ is forced to assume a value equal to $1$ and the value of $t$ is then totally incorporated into $\beta_{K+1}$. In the following we will assume for ease of notation that the problem is given in homogeneous coordinates and that the constants $M$ and $K$ already account for the additional single-variable group.

\begin{definition}
The partitioned least square (\PLS) problem is formulated as:
\begin{gather*}
\text{minimize}_{\vec{\alpha}, \vec{\beta}} \| \mat{X} \times (\mat{P} \circ \vec{\alpha}) \times \vec{\beta} - \vec{y} \|_2^2 \\
\begin{aligned}
\textup{s.t.}\quad  &\vec{\alpha}  \succeq 0\\
                    &\T{\mat{P}} \times \vec{\alpha} = \vec{1}.
\end{aligned}
\end{gather*}
\end{definition}
\noindent In summary, we want to minimize the squared residuals of $f(\mat{X})$, as defined in (\ref{eq:f-matrix-form}), under the constraint that for each subset $k$ in the partition, the set of weights form a distribution: they need to be all nonnegative as imposed by $\vec{\alpha} \succeq 0$ constraint and they need to sum to $1$ as imposed by constraint $\T{\mat{P}} \times \vec{\alpha} = \vec{1}$.

Unfortunately we do not know a closed form solution for this problem. Furthermore, the problem is not convex and hence hard to solve to global optimality using standard out-of-the-box solvers. Even worse, later on we shall prove that the problem is actually NP-complete. The following theorem states the non-convexity of the objective function formally. 
\vspace{1em}

\begin{theorem}
  The \PLS problem is not convex.
\end{theorem}

\begin{proof}
  It suffices to show that the Hessian of the objective function is not positive semidefinite. 
  The detasils of the proof can be found in \cite{esposito2019partitioned}.

\end{proof}

In the following we will provide two algorithms that solve the above problem. One is an alternating least squares approach which scales well with $K$, but it is not guaranteed to provide the globally optimal solution. The other one is a reformulation of the problem through a (possibly) large number of convex problems whose minimum is guaranteed to be the globally optimal solution of the original problem. Even though the second algorithm does not scale well with $K$, we believe that this \emph{should not be a problem} since the \PLS is by design well suited for a small group of interpretable groups. However, we do sketch a possible branch and bound strategy to mitigate this problem in Section~\ref{sec:branch-and-bound}.

\begin{remark}
The \PLS model presented so far has no regularization mechanism in place and, as such, it risks overfitting the training set. Since the $\alpha$ values are normalized by definition, the only parameters that need regularization are those collected in the $\vec{\beta}$ vector.
Then, the regularized version of the objective function simply adds a penalty on the size of the $\vec{\beta}$ vector:
\begin{equation}
  \label{eq:reg-objective}
  \| \mat{X} \times (\mat{P} \circ \vec{\alpha}) \times \vec{\beta} \|^2_2 + \eta \|\vec{\beta}\|^2_2, 
\end{equation}
where the squared euclidean norm could be substituted with the L1 norm in case a LASSO-like regularization is preferred.
\end{remark}
 \section{Algorithms}
\label{sec:algorithms}

\subsection{Alternating Least Squares approach}
\label{ssec:als}

In the \PLS problem we aim at minimizing a non-convex objective, where the non-convexity depends on the multiplicative interaction between $\alpha$ and $\beta$ variables in the expression $\|  \mat{X} \times (\mat{P} \circ \vec{\alpha}) \times \vec{\beta} - \vec{y}\|_2^2$. Interestingly, if one fixes $\vec{\alpha}$, the expression $ \mat{X} \times (\mat{P} \circ \vec{\alpha})$ results in a matrix $\mat{X}'$ that does not depend on any variable. Then, the whole expression can be rewritten as a problem $p_\vec{\alpha}$ whose objective function $\| \mat{X}' \vec{\beta} - \vec{y}\|_2^2$ depends on the parameter vector $\vec{\alpha}$ and is the convex objective function 
of a standard least squares problem in the $\beta$ variables. In a similar way, it can be shown that by fixing $\vec{\beta}$ one also ends up with a convex optimization problem $p_\vec{\beta}$.
Indeed, after fixing $\vec{\beta}$, the objective function is the squared norm of a vector whose components are affine functions of vector $\vec{\alpha}$ (see Section~\ref{sec:numerical-stability} for more details).
These observations naturally lead to the formulation of an alternating least squares solution where one alternates between solving $p_\vec{\alpha}$ and $p_\vec{\beta}$. In Algorithm~\ref{algo:pls-a} we formalize this intuition into the \algoalt function where, after  initializing $\vec{\alpha}$ and $\vec{\beta}$ randomly, we iterate until some stopping criterion is satisfied (in our experiments we fixed a number $T$ of iterations, but one may want to stop the algorithm as soon as $\vec{a}$ and $\vec{c}$ do not change between two iterations). At each iteration we take the latest estimate for the $\alpha$ variables and solve the $p_\vec{\alpha}$ problem based on that estimate, we then keep the newly found $\beta$ variables and solve the $p_\vec{\beta}$ problem based on them. At each iteration the overall objective is guaranteed not to increase in value and, indeed, we prove that, if the algorithm is never stopped, the sequence of $\vec{\alpha}$ and $\vec{\beta}$ vectors found by \algoalt has at least one accumulation point and that all accumulation points are partial optima\footnote{A partial optima of a function $f(\vec{\alpha}, \vec{\beta})$ is a point $(\vec{\alpha}^\star, \vec{\beta}^\star)$ such that $\forall \vec{\alpha} : f(\vec{\alpha}^\star, \vec{\beta}^\star) \leq f(\vec{\alpha}, \vec{\beta}^\star)$ and $\forall \vec{\beta} : f(\vec{\alpha}^\star, \vec{\beta}^\star) \leq f(\vec{\alpha}^\star, \vec{\beta})$.} with the same function value. 

\begin{theorem}
  Let $\vec{\zeta}_i = (\vec{\alpha}_i, \vec{\beta}_i)$ be the sequence of $\vec{\alpha}$ and $\vec{\beta}$ vectors found by \algoalt to the \PLS problem and assume that the objective function is regularized as described in (\ref{eq:reg-objective}), then:
  \begin{enumerate}
    \item the sequence of $\vec{\zeta}_i$ has at least one accumulation point, and
    \item all accumulation points are partial optima  attaining the same value of the objective function.
  \end{enumerate}
\end{theorem}

\begin{proof}
The \PLS problem is actually a biconvex optimization problem and Algorithm~\ref{algo:pls-a} is actually a specific instantiation of the \emph{Alternating Convex Search} strategy~\citep{gorski07biconvex} to solve biconvex problems. Theorem~4.9 in \citep{gorski07biconvex} implies that:
\begin{itemize}
  \item \textbf{if} the sequence $\vec{\zeta}_i$ is contained in a compact set \textbf{then} it has at least one accumulation point, and
  \item \textbf{if} for each accumulation point $\vec{\zeta}^\star$ of the sequence $\vec{\zeta}_i$, either the optimal solution of the problem with fixed $\vec{\alpha}$ is unique, or the optimal solution of the problem with fixed $\vec{\beta}$ is unique; \textbf{then} all accumulation points are partial optima and have the same function value.
\end{itemize}

The first requirement is fulfilled in our case since $\vec{\alpha}$ is constrained by definition into $[0,1]^{M}$, while the regularization term prevents $\vec{\beta}$ from growing indefinitely. The second requirement is fulfilled since for fixed $\vec{\alpha}$ the optimization function is quadratic and strictly convex in $\vec{\beta}$. Hence, the solution is unique.
\end{proof}

\begin{lstlisting}[language=Julia,escapeinside={(@}{@)},numbers=left,float,
  caption=Alternating least squares solution to the \PLS problem. The notation $\mathtt{const}$($\vec{\alpha}$) (respectively $\mathtt{const}$($\vec{\beta}$)) is just to emphasize that the current value of $\vec{\alpha}$ (respectively $\vec{\beta}$) will be used as a constant in the following step. ,label=algo:pls-a]
  function (@\algoalt@)((@$\mat{X}, \vec{y}, \mat{P}$@))
    (@$\vec{\alpha}$@) = random(M)
    (@$\vec{\beta}$@) = random(K)

    for (@$t$@) in (@$ 1 \dots T$@)
      (@$\vec{a}$@) = const((@$\vec{\alpha}$@))
      (@$p^\star$@) = minimize(@$_\vec{\beta}$@)((@$\| (\mat{X} \times (\mat{P} \circ \vec{a}) \times \vec{\beta} - \vec{y} \|_2^2)$@))(@\label{algo:pls-a-min-alpha}@)


      (@$\vec{c}$@) = const((@$\vec{\beta}$@))
      (@$p^\star$@) = minimize(@$_\vec{\alpha}$@)((@$\| (\mat{X} \times (\mat{P} \circ \vec{\alpha}) \times \vec{c} - \vec{y} \|_2^2$@),
              (@$\vec{\alpha} \succeq 0$@),
              (@$\T{\mat{P}} \times \vec{\alpha} = 1$@))(@\label{algo:pls-a-min-beta}@)
    end
  
    return (@$(p^\star, \vec{\alpha}, \vec{\beta} )$@)
  end
\end{lstlisting}

\subsection{Reformulation as a set of convex subproblems}
\label{ssec:pls-b}

Here we show how the \PLS problem can be reformulated as a new problem with binary variables which, in turn, can be split into a set of convex problems 
such that the smallest objective function value among all local (and global) minimizers of these convex problems is also the global optimum value of the \PLS problem.
\begin{definition} The \emph{\PLS-b problem} is a \PLS problem in which the $\beta$ variables are substituted by a binary variable vector $\vec{b} \in \{-1,1\}^K$, and the normalization constraints over the $\alpha$ variables are dropped:
  \begin{gather*}
  \text{minimize}_{\vec{\alpha}, \vec{b}} \| \mat{X} \times (\mat{P} \circ \vec{\alpha}) \times \vec{b} - \vec{y} \|_2^2 \\
  \begin{aligned}
  \textup{s.t.}\quad  &\vec{\alpha}  \succeq 0\\
\quad & \vec{b} \in \{-1,1\}^K. \\
  \end{aligned}
  \end{gather*}
\end{definition}
The \PLS-b problem turns out to be a Mixed Integer Nonlinear Programming (MINLP) problem with a peculiar structure. More specifically, we note that the above definition actually defines $2^K$ minimization problems, one for each of the possible instances of vector $\vec{b}$. Interestingly, each one of the minimization problems can be shown to be convex by the same argument used in Section~\ref{ssec:als} (for fixed $\beta$ variables) and  we will prove that the minimum attained by minimizing all sub-problems corresponds to the global minimum of the original problem. We also show that by simple algebraic manipulation of the result found by a \PLS-b solution, it is possible to write a corresponding \PLS solution attaining the same objective. 

The main breakthrough here derives from noticing that in the original formulation the $\beta$ variables are used to keep track of two facets of the solution: \emph{i)} the magnitude and \emph{ii)} the sign of the contribution of each subset in the partition of the variables. With the $\vec{b}$ vector keeping track of the signs, one only needs to reconstruct the magnitude of the $\beta$ contributions to recover the solution of the original problem.

The following theorem states the equivalence between the \PLS and the \PLS-b problem. More precisely, we will prove that for any feasible solution of one of the two problems, one can build a feasible solution of the other problem with the same objective function value, from which equality between the optimal values of the two problems immediately follows. 
\begin{theorem}
Let $(\vec{\alpha}, \vec{b})$ be a feasible solution of the \PLS-b problem. Then, there exists a feasible solution $(\vec{\normalpha},\vec{\normbeta})$ of the \PLS problem such that:
\begin{equation}
\label{eq:eqobj}
 \| \mat{X} \times (\mat{P} \circ \vec{\alpha}) \times \vec{b} - \vec{y} \|_2^2=\|  \mat{X} \times (\mat{P} \circ \vec{\normalpha}) \times \vec{\normbeta} - \vec{y}\|_2^2.
\end{equation}
Analogously, for each feasible solution $(\vec{\normalpha},\vec{\normbeta})$ of the \PLS problem, there exists a feasible solution $(\vec{\alpha}, \vec{b})$ of the \PLS-b problem such that (\ref{eq:eqobj}) holds.
Finally, $p^\star = p_b^\star$, where $p^\star$ and $p_b^\star$ denote, respectively, the optimal value of the \PLS problem and of the \PLS-b problem.
\end{theorem}
\begin{proof}
Let $(\vec{\alpha}, \vec{b})$ be a feasible solution of the \PLS-b problem and let
$\avgbeta$ be a normalization vector containing in $\avgbetaat{k}$ the normalization factor for variables in partition subset $k$:
\[
  \avgbeta = \left( \sum_{m \in P_k} \alpha_m \right)_k = \T{\mat{P}} \times \vec{\alpha}.
\]
Then, for each $m$ such that  $\avgbetaat{k[m]}\neq 0$, we define $\hat{\alpha}_m$ as follows:
\[
\hat{\alpha}_m=\frac{\alpha_m}{\avgbetaat{k[m]}},
\]
while for any $m$ such that $\avgbetaat{k[m]}= 0$ we can define  $\hat{\alpha}_m$, e.g., as follows:
\[
\hat{\alpha}_m=\frac{1}{|P_{k[m]}|}.
\]
In fact, for any $k$ such that $\avgbetaat{k}= 0$, any definition of $\hat{\alpha}_m$ for $m\in P_k$ such that $\sum_{m\in P_k}  \hat{\alpha}_m=1$ would be acceptable.
The $\normbeta$ vector can be reconstructed simply by taking the Hadamard product of $\vec{b}$ and $\avgbeta$:
\[
  \normbeta = \vec{b} \circ \avgbeta.
\]
In order to prove (\ref{eq:eqobj}), we only need to prove that
 \[
    \mat{X} \times (\mat{P} \circ \vec{\alpha}) \times \vec{b} = \mat{X} \times (\mat{P} \circ \normalpha) \times \normbeta.
  \]
The equality is proved as follows:
\begin{align*}
  \mat{X} \times (\mat{P} \circ \normalpha) \times \normbeta & = &
  \mat{X} \times \left(\mat{P} \circ \left(\frac{\alpha_{m}}{\avgbetaat{k[m]}}\right)_m\right) \times \left( b_k \avgbetaat{k} \right)_k \\
    & = &
    \left( \sum_k b_k \avgbetaat{k} \sum_{m \in P_k} \frac{\alpha_{m}}{\avgbetaat{k[m]}} x_{n,m} \right)_n \\
    & = &
    \left( \sum_k b_k \avgbetaat{k} \sum_{m \in P_k} \frac{\alpha_{m}}{\avgbetaat{k}} x_{n,m} \right)_n  \\
    & = &
    \left( \sum_k b_k \sum_{m \in P_k} \alpha_{m} x_{n,m} \right)_n \\
    & = &
    \mat{X} \times (\mat{P} \circ \vec{\alpha}) \times \vec{b},
\end{align*}
where in between row 2 and row 3 we used the fact that $\avgbetaat{k}$ and $\avgbetaat{k[m]}$ are two ways to write the same thing (the former using directly the partition number $k$, and the latter using the notation $k[m]$ to get the partition number from the feature number $m$). To be more precise, we only considered the case when $\avgbetaat{k[m]}\neq 0$ for all $m$. But the result can be easily extended to the case when  $\avgbetaat{k[m]}= 0$ for some $m$, by observing that in this case the corresponding terms give  a null contribution to both sides of the equality.
\newline\newline\noindent
Now, let  $(\vec{\normalpha},\vec{\normbeta})$ be a feasible solution of the \PLS problem. Then, we can build a feasible solution $(\vec{\alpha}, \vec{b})$ for the \PLS-b problem as follows.
For any $k\in \{1,\ldots,K\}$ let:
$$
b_k=
\left\{
\begin{array}{ll}
-1 & \mbox{if}\ \hat{\beta}_k< 0 \\ 
+1 & \mbox{otherwise,}
\end{array}
\right.
$$
while for each $m$, let:
$$
\alpha_m=b_{k[m]} \hat{\beta}_{k[m]} \hat{\alpha}_m.
$$
Equivalence between the objective function values at $(\vec{\normalpha},\vec{\normbeta})$ and $(\vec{\alpha}, \vec{b})$ is proved in a way completely analogous to what we have seen before.
\newline\newline\noindent
Finally, the equivalence between the optimal values of the two problems is an immediate corollary of the previous parts of the proof. In particular, it is enough to observe that for any optimal solution
of one of the two problems, there exists a feasible solution of the other problem with the same objective function value, so that both $p^\star \geq p_b^\star$ and $p^\star \leq p_b^\star$ holds, and, thus, $p^\star = p_b^\star$.
\end{proof}
The complete algorithm, which detects and returns the best solution of the \PLS-b problems by iterating over all possible vectors $\vec{b}$, is implemented by the function \algoopt reported in Algorithm~\ref{algo:pls-b}.
\begin{lstlisting}[language=Julia,escapeinside={(@}{@)},numbers=left,float,caption=\PLS-b solution to the \PLS problem. {The function $\mathtt{extract\_min}$ retrieves the $(\dot{p}, \dot{\vec{\alpha}}, \dot{\vec{\beta}})$ tuple in the results array attaining the lowest $\dot{p}$ value.},label=algo:pls-b]
function (@\algoopt@)((@$\mat{X}, \vec{y}, \mat{P}$@))
  results = []

  for (@$\vec{\dot{b}}$@) in (@$\{1,-1\}^K$@)
    (@$\dot{p}$@) = minimize(@$_{\dot{\vec{\alpha}}}$@)((@$\| (\mat{X} \times (\mat{P} \circ \dot{\vec{\alpha}}) \times \dot{\vec{b}} - \vec{y} \|_2^2)$@),(@ $\dot{\vec{\alpha}} \succeq 0$@)) (@\label{algo:pls-b-min-alpha}@)

    results += ((@$\dot{p},\dot{\vec{\alpha}},\dot{\vec{b}}$@))
  end

  (@$p^\star,\vec{\alpha},\vec{b}$@) = extract_best(results)


  (@$\avgbeta$@) = (@$\T{\mat{P}} \times \vec{\alpha}$@)
  (@$\normalpha$@) = (@$(\mat{P} \circ \vec{\alpha} \oslash  \T{\avgbeta}) \times \vec{1}$@)
  (@$\normbeta$@) = (@$\vec{b} \circ \vec{\avgbeta}$@)

  return (@$(p^\star, \normalpha, \normbeta)$@)
end
\end{lstlisting}
\begin{remark}
When dealing with the \PLS-b problem, the regularization term introduced for the objective function of the \PLS problem, reported in (\ref{eq:reg-objective}), needs to be slightly updated so to accommodate the differences in the objective function when used in Algorithm~\ref{algo:pls-b}. In this second case, since the $\beta$ variables do not appear in the optimization problems obtained  after fixing the different binary vectors $\vec{b}$, the regularization term $\| \vec{\beta} \|^2_2$ is replaced by $\|\T{\mat{P}} \times \vec{\alpha}\|_2^2$. We notice that since the new regularization term is still convex, it does not hinder the convexity of the optimization problems.
\end{remark}

\subsection{Numerical Stability}
\label{sec:numerical-stability}
The optimization problems solved within Algorithms~\ref{algo:pls-a} and \ref{algo:pls-b}, despite being convex, are sometimes hard to solve due to numerical problems. General-purpose solvers often find the data matrix to be ill-conditioned and return sub-optimal results~\cite{bjorck1996numerical,cucker2007mixed}. In this section we show how to rewrite the problems so to mitigate these difficulties. The main idea is to recast the minimization problems as standard least squares and non-negative least squares problems, and to employ efficient solvers for these specific problems rather than the general-purpose ones. 

We start by noticing that the minimization problem at line~\ref{algo:pls-a-min-alpha} of Algorithm~\ref{algo:pls-a} can be easily solved by a standard least square algorithm since the expression $\mat{X} \times (\mat{P} \circ \vec{a} )$ computes to a constant matrix $\mat{X}'$ and the original problem simplifies to the ordinary least squares problem: $\text{minimize}_\vec{\beta}(\|\mat{X}' \vec{\beta} - \vec{y}\|_2^2)$. 

For what concerns the minimization problem at line~\ref{algo:pls-a-min-beta} of the same algorithm, we notice that we can initially ignore the constraint $\mat{P}^\top \times \vec{\alpha} = 1$.
Without such constraint, the problem turns out to be a non-negative least squares problem. Indeed, we note that expression $\mat{X} \times (\mat{P} \circ \vec{\alpha})\times \vec{c}$ can be rewritten as the constant matrix $\mat{X} \circ (\mat{P} \circ \vec{c}^\top \times \vec{1})^\top$ multiplied by the vector $\vec{\alpha}$, so that the whole minimization problem could be rewritten as:
\begin{gather*}
 \text{minimize}_{\vec{\alpha}} \| \mat{X} \circ (\mat{P} \circ \vec{c}^\top \times \vec{1})^\top \times \vec{\alpha} - \vec{y}\|_2^2 \\
 \begin{aligned}
  \textup{s.t.}\quad  &\vec{\alpha}  \succeq 0.\\
  \end{aligned}
\end{gather*}
After such problem has been solved, the solution of the problem including the constraint $\mat{P}^\top \times \vec{\alpha} = 1$ can be easily obtained by dividing each $\alpha$ subset by a normalizing factor and multiplying the corresponding $\beta$ variable by the same normalizing factor (it is the same kind of operations we exploited in Section~\ref{ssec:pls-b}; in that context the normalizing factors were denoted with $\avgbeta$).

In a completely analogous way we can rewrite the minimization problem at line~\ref{algo:pls-b-min-alpha} of Algorithm~\ref{algo:pls-b} as:
\begin{equation}
\begin{gathered}
  \label{eq:pls-b-non-neg}
  \text{minimize}_{\dot{\vec{\alpha}}} \| \mat{X} \circ (\mat{P} \circ \dot{\vec{b}}^\top \times \vec{1})^\top \times \dot{\vec{\alpha}} - \vec{y}\|_2^2 \\
  \begin{aligned}
  \textup{s.t.}\quad  &\dot{\vec{\alpha}}  \succeq 0,\\
  \end{aligned}
\end{gathered}
\end{equation}
which, again, is a non-negative least squares problem.

As previously mentioned, by rewriting the optimization problems as described above and by employing special-purpose solvers for the least squares and the non-negative least squares problems, solutions appear to be more stable and accurate.

\begin{remark}
  Many non-negative least squares solvers do not admit an explicit regularization term. An $l_2$-regularization term equivalent to $\rho \|\vec{\beta}\|_2^2 = \rho \| \mat{P}^\top \times \vec{\alpha} \|_2^2 = \rho \sum_k (\sum_{m \in P_k} \alpha_m)^2$ can be implicitly added by augmenting the data matrix $\mat{X}$ with $K$ additional rows. The trick is done by setting all the additional $y$ to $0$ and the $k$-th additional row as follows:
\[
  x_{N+k,m} =
  \begin{cases}
    \sqrt{\rho} & \mbox{if } m \in P_k \\
    0 & \mbox{otherwise.}
  \end{cases}
\]
When the additional $k$-th row and the additional $y$ are plugged into the expression inside the norm in (\ref{eq:pls-b-non-neg}), the expression evaluates to:
\[
  \sum_{m \in P_k} \sqrt{\rho} \ \dot{b}_k \alpha_m - 0 = \dot{b}_k  \sqrt{\rho}\sum_{m \in P_k} \alpha_m,
\]
which reduces to $\rho \sum_k (\sum_{m \in P_k} \alpha_m)^2$ when squared and summed over all the $k$ as a result of the evaluation of the norm.

\end{remark} \subsection{An alternative branch-and-bound approach}
\label{sec:branch-and-bound}
Algorithm \ref{algo:pls-b} is based on a {\em complete enumeration} of all possible $2^K$ vectors $\vec{b}$. Of course, such an approach becomes too expensive as soon as $K$ gets large. As already previously commented,  \algoopt is by design well suited for small $K$ values, so that complete enumeration should be a valid option most of the times. However, for the sake of completeness, in this section we discuss a branch-and-bound approach, based on {\em implicit enumeration}, which could be employed as $K$ gets large. Pseudo-code detailing the approach is reported in Algorithm~\ref{algo:pls-bnb}.

\begin{lstlisting}[language=Julia,escapeinside={(@}{@)},numbers=left,float,
  caption={Pseudo code for a depth-first implementation of the branch and bound optimization of \algoopt. \ub is the current upper bound of the optimal value, $\Sigma$ is the set of constraints associated to the current node. For the sake of simplicity, the algorithm returns only the optimal values. It is easy to modify it to keep track of the best solution as well. \texttt{lower\_bound} computes the relaxation of either~(\ref{eq:quadform}) or~(\ref{eq:quadform2}) subject to the constraints in $\Sigma$ and returns the lower bound $lb$ (the lower bound itself) and $\vec{\alpha}$ (the values of the variables attaining it).}, label=algo:pls-bnb]
  function PartLS-bnb((@$\mat{X}, \vec{y}, \mat{P}, \ub, \Sigma$@))
    (@$\mat{Q}$@) = (@ $\mat{X}^\top \mat{X}$ @)
    (@$\vec{q}$@) = (@$ -2 \mat{X}^\top\vec{y} $@)
    (@$q_0$@) = (@$ \vec{y}^\top\vec{y} $@)

    (@$ lb, \vec{\alpha} $@) = lower_bound((@$\mat{Q}, \vec{q}, q_0, \Sigma$@))
    if (@$lb \geq \ub$@)
      # No solution better than the upper 
      # bound can be found in the current 
      # branch
      return (@$ \ub $@)   
    end

    (@$ \vec{\nu} = \left(\sum_{i,j \in P_k} \max\{0, -\alpha_i \alpha_j\} \right)_k $@)
    if (@$\vec{\nu} = \vec{0}$@)
      # optimal solution found for this branch
      # we can avoid further splitting
      # constraints
      return (@$ lb $@)
    end

    (@$k = \arg\max_k v_k$ @)

    (@$\Sigma_+$@) = (@$ \Sigma \cup \{ \forall i \in P_k: \alpha_i \geq 0 \} $@)
    (@$\Sigma_-$@) = (@$ \Sigma \cup \{ \forall i \in P_k: \alpha_i \leq 0 \} $@)

    (@$\mu_+ $@) = PartLS-bnb(@($\mat{X}, \vec{y}, \mat{P}, \ub, \Sigma_+ )$@)
    (@$\mu_- $@) = PartLS-bnb(@($\mat{X}, \vec{y}, \mat{P}, \min(\ub, \mu_+), \Sigma_- )$@)

    return (@$ \min(\ub, \mu_+, \mu_-) $@)
  end
\end{lstlisting}

First, we remark that the \PLS-b problem can be reformulated as follows
\begin{equation}
\label{eq:refplsb1}
\begin{array}{ll}
  \text{minimize}_{\vec{\alpha}} & \sum_n \left(\sum_k \sum_{m\in P_k} \alpha_m x_{n,m} -y_n\right)^2 \\
  \textup{s.t.}\quad & \alpha_i \alpha_j \geq 0\ \ \ \forall i,j\in P_k,\ \ \ \forall k\in \{1,\ldots,K\},
  \end{array}
\end{equation}
where we notice that vector $\vec{b}$ and the nonnegativity constraints $\vec{\alpha}  \succeq 0$ have been eliminated, and replaced by the new constraints, which impose that for any $k$, all variables $\alpha_m$ such that $m\in P_k$ must have the same sign.   The new problem is a quadratic one with a convex quadratic objective function and simple (but non-convex) bilinear constraints. We note that, having removed the $\vec{b}$ variables, the scalar objective do not need the distinction between groups anymore and it can rewritten as 
$\sum_n \left( \sum_m \alpha_m x_{n,m}  - y_n\right)^2$ or, in matrix form, as 
$\| \mat{X} \vec{\alpha} -y \|^2 = (\mat{X} \vec{\alpha} - \vec{y})^\top (\mat{X} \vec{\alpha} - \vec{y})$. 
Hence, we can reformulate the problem as follows
\begin{equation}
\label{eq:quadform}
\begin{array}{ll}
\text{minimize}_{\vec{\alpha}} & \vec{\alpha}^\top\mat{Q}\vec{\alpha}+\vec{q}^\top \vec{\alpha} + q_0 \\
\textup{s.t.}\quad & \alpha_i \alpha_j \geq 0\ \ \ \forall i,j\in P_k,\ \ \ \forall k\in \{1,\ldots,K\},
\end{array}
\end{equation}
where $Q = \mat{X}^\top \mat{X}$, $\vec{q} = -2\mat{X}^\top \vec{y}$, and $q_0 = \vec{y}^\top \vec{y}$.
Different lower bounds for this problem can be computed. The simplest one is obtained by simply removing all the constraints, which results in an unconstrained convex quadratic problem.
A stronger, but more costly, lower bound can be obtained by solving the classical semidefinite relaxation of quadratic programming problems. First, we observe that problem (\ref{eq:quadform}) can be rewritten as follows (see \citep{Shor87})
\begin{equation}
\label{eq:quadform2}
\begin{array}{ll}
\text{minimize}_{\vec{\alpha}, \mat{A}} & \mat{Q}\bullet \mat{A} +\vec{q}^\top \vec{\alpha} + q_0 \\
\textup{s.t.}\quad  & \mat{A}_k=\vec{\alpha}_{P_k}\vec{\alpha}_{P_k}^\top\ \ \ \forall k\in \{1,\ldots,K\} \\
& \mat{A}_k \geq \mat{O}\ \ \ \forall k\in \{1,\ldots,K\},
\end{array}
\end{equation}
where $\mat{Q}\bullet \mat{A}=\sum_{i,j} Q_{ij} A_{ij}$, and $\vec{\alpha}_{P_k}$ is the restriction of $\vec{\alpha}$ to the entries in $P_k$, $k\in\{1,\ldots,K\}$.
Next, we observe that the equality constraint $\mat{A}=\vec{\alpha}\vec{\alpha}^\top$ is equivalent to requiring that $\mat{A}$ is a psd (positive semidefinite) matrix {\em and} is of rank one. If we remove the (non-convex) rank one requirement, we end up with the following convex relaxation of (\ref{eq:quadform}) requiring the solution of a semidefinite programming problem:
$$
\begin{array}{ll}
\text{minimize}_{\vec{\alpha}, \mat{A}} & \mat{Q}\bullet \mat{A} +\vec{q}^\top \vec{\alpha} + q_0 \\
\textup{s.t.}  & 
\mat{A}_k-\vec{\alpha}_{P_k}\vec{\alpha}_{P_k}^\top\ \ \mbox{is psd} \ \ \ \forall k\in \{1,\ldots,K\} \\
& \mat{A}_k \geq \mat{O}\ \ \ \forall k\in \{1,\ldots,K\}.
\end{array}
$$
Note that by Schur complement, constraint ``$\mat{A}_k-\vec{\alpha}_{P_k}\vec{\alpha}_{P_k}^\top\ \ \mbox{is psd}$'' is equivalent to the following semidefinite constraint:
$$
\left(
\begin{array}{cc}
1 & \vec{\alpha}_{P_k}^\top \\
\vec{\alpha}_{P_k} & \mat{A}_k
\end{array}
\right) \ \ \mbox{is psd.}
$$
No matter which problem we solve to get a lower bound, after having solved it we can consider the vector $\vec{\alpha}^\star$ of the optimal values of the $\alpha$ variables at its optimal solution and we can compute the following quantity for each
$k\in \{1,\ldots,K\}$
$$
\nu_k=\sum_{i,j\in P_k} \max\{0,-\alpha_i^\star\alpha_j^\star\}.
$$
If $\nu_k=0$ for all $k$, then the optimal solution of the relaxed problem is feasible and also optimal for the original problem (\ref{eq:quadform}) and we are done. Otherwise, we can select an index $k$ such that $\nu_k>0$ (e.g., the largest one, corresponding to the largest violation of the constraints), and split the original problem into two subproblems, one where we impose that all variables $\alpha_m$, $m\in P_k$, are nonnegative, and the other where we impose that all variables $\alpha_m$, $m\in P_k$, are nonpositive. Lower bounds for the new subproblems can be easily computed by the same convex relaxations employed for the original problem (\ref{eq:quadform}), but with the additional constraints. The violations $\nu_k$ are computed also for the subproblems and, in case one of them is strictly positive, the corresponding subproblem may be further split into two further subproblems, unless its lower bound becomes at some point larger than or equal to the current global upper bound of the problem, which is possibly updated each time a new feasible solution of   
(\ref{eq:quadform}) is detected. As previously commented, Algorithm~\ref{algo:pls-bnb} provides a possible implementation  of the branch-and-bound approach. More precisely, Algorithm~\ref{algo:pls-bnb} is an implementation where nodes of the branch-and-bound tree are visited in a depth-first manner. An alternative implementation is, e.g., the one where nodes are visited in a lowest-first manner, i.e., the first node to be visited is the one with the lowest lower bound.
 \section{Complexity}
In this  section we establish the theoretical complexity of the \PLS-b problem. In view of reformulation (\ref{eq:refplsb1}), it is immediately seen that the cases where $|P_k|=1$ for all $k=1,\ldots,K$, are polynomially solvable. Indeed, in this situation problem  (\ref{eq:refplsb1}) becomes unconstrained and has a convex quadratic objective function. Here we prove that as soon as we move from $|P_k|=1$ to $|P_k|=2$, the problem becomes NP-complete. We prove this by showing that each instance of the NP-complete problem \subsetsum (see, e.g., \citep{garey1979computers}) can be transformed in polynomial time into an instance of problem  (\ref{eq:refplsb1}).
We recall that problem  \subsetsum is defined as follows. Let $s_1,\ldots,s_k$ be a collection of $K$ positive integers. We want to establish whether there exists a partition of this set of integers into two subsets such that the sums of the integers belonging to the two subsets is equal, i.e., whether there exist $I_1, I_2\subseteq \{1,\ldots,K\}$ such that:
\begin{equation} 
\label{eq:subsetsum}
I_1\cup I_2=\{1,\ldots,K\},\ I_1\cap I_2=\emptyset,\ \  \sum_{k\in I_1} s_k = \sum_{k\in I_2} s_k.
\end{equation}

Now, let us consider an instance of problem (\ref{eq:refplsb1}) with $K$ partitions and two variables $\alpha_{m_{1,k}}$ and $\alpha_{m_{2,k}}$ for each partition $k$ (implying $M=2K$). The data matrix $\mat{X}$ and vector $\vec{y}$ have $N=3K+1$ rows  defined as follows (when $k$ and $m$ are not restricted, they are assumed to vary on $\{1\ldots K\}$ and $\{1 \ldots M\}$ respectively):
\[
\begin{array}{ll}
    x_{k,m_{1,k}}=1,\\ 
    x_{k,m_{2,k}}=-1,\\ 
    x_{k,m}=0,\ y_k=-s_k & m \not \in \{m_{1,k}, m_{2,k}\}  \vspace{0.5cm}\\
    x_{K+k,m_{1,k}}=\sqrt{\rho},
    x_{K+k,m_{2,k}}=0, \\
    x_{K+k,m}=0, \\
    y_{K+k}=0 & m \not \in \{m_{1,k}, m_{2,k}\}  \vspace{0.5cm} \\
    x_{2K+k,m_{1,k}}=0,\\
    x_{2K+k,m_{2,k}}=\sqrt{\rho},\\
    x_{2K+k,m}=0,\ y_{2K+k}=0  & m \not \in \{m_{1,k}, m_{2,k}\} \vspace{0.5cm}\\
    x_{3K+1,m}=1,\ y_{3K+1}=0. \ 
\end{array}
\]
When the values so defined are plugged into problem (\ref{eq:refplsb1}) we obtain:
\begin{equation}
\begin{aligned}
\label{eq:polytrans}
\text{minimize}_{\vec{\alpha}} & \sum_{k=1}^K (\alpha_{m_{1,k}}-\alpha_{m_{2,k}}-s_k)^2 + \rho\sum_{k=1}^K (\alpha_{m_{1,k}})^2 +  \\ & \quad \rho \sum_{k=1}^K (\alpha_{m_{2,k}})^2 + \left[\sum_{k=1}^K (\alpha_{m_{1,k}}+\alpha_{m_{2,k}})\right]^2  \\ 
\text{s.t.} \quad & \alpha_{m_{1,k}}\alpha_{m_{2,k}}\geq 0 \quad \forall k \in \{1,\ldots,K\},
\end{aligned}
\end{equation}
with $\rho>0$. 

We prove the following theorem, which states that an instance of
the \subsetsum problem (\ref{eq:subsetsum}) can be solved by solving the corresponding instance (\ref{eq:polytrans}) of  problem (\ref{eq:refplsb1}), and, thus, establishes NP-completeness of the \PLS-b problem.
\begin{theorem}
The optimal value of (\ref{eq:polytrans}) is equal to 
$$
\frac{\rho \sum_{k=1}^K s_k^2}{1+\rho},
$$
if and only if there exist $I_1, I_2$
such that (\ref{eq:subsetsum}) holds, i.e., if and only if
the \subsetsum problem admits a solution.
\end{theorem}
\begin{proof}
As a first step we derive the optimal solutions of the following restricted two-dimensional problems for $k\in \{1,\ldots,K\}$:
\begin{equation}
\label{eq:restricted}
\begin{array}{lll}
\text{minimize}_{\alpha_{m_{1,k}}, \alpha_{m_{2,k}}} &  (\alpha_{m_{1,k}}-\alpha_{m_{2,k}}-s_k)^2 + \\
& \quad \rho (\alpha_{m_{1,k}})^2+\rho (\alpha_{m_{2,k}})^2  & \\
\text{s.t.} & \alpha_{m_{1,k}}\alpha_{m_{2,k}}\geq 0. &
\end{array}
\end{equation}
This problems admits at least a global minimizer since its objective function is strictly convex quadratic. 
Global minimizers should be searched for among {\em regular} KKT points and {\em irregular} points. Regular points are those who fulfill a {\em constraint qualification}. In particular, in this problem all feasible points, except the origin, fulfill
the constraint qualification based on the linear independence of the gradients of the active constraints. This is trivially true since there is a single constraint and the gradient of such constraint is null only at the origin.
Thus, the only irregular point is the origin. In order to detect the KKT points, we first write down
the KKT conditions:
$$
\begin{array}{l}
2(\alpha_{m_{1,k}}-\alpha_{m_{2,k}}-s_k)+ 2 \rho \alpha_{m_{1,k}} -\mu \alpha_{m_{2,k}}=0 \\
-2(\alpha_{m_{1,k}}-\alpha_{m_{2,k}}-s_k)+ 2 \rho \alpha_{m_{2,k}} -\mu \alpha_{m_{1,k}}=0 \\
\alpha_{m_{1,k}}\alpha_{m_{2,k}}\geq 0 \\
\mu \alpha_{m_{1,k}}\alpha_{m_{2,k}}=0 \\
\mu\geq 0,
\end{array}
$$
where $\mu$ is the Lagrange multiplier of the constraint.
We can enumerate all KKT points of problem (\ref{eq:restricted}). By summing up the first two equations, we notice that
$$
(\mu -2\rho)(\alpha_{m_{1,k}}+\alpha_{m_{2,k}})=0,
$$
must hold. This equation is satisfied if:
\begin{itemize}
\item either $\alpha_{m_{1,k}}+\alpha_{m_{2,k}}=0$, which implies $\alpha_{m_{1,k}}=\alpha_{m_{2,k}}=0$,  in view of $\alpha_{m_{1,k}}\alpha_{m_{2,k}}\geq 0$. As previously mentioned, the origin is the unique irregular point. So, it is not a KKT point but when searching
for the global minimizer, we need
to compute the objective function value also at such point and this is equal to $s_k^2$;
\item or $\mu=2\rho>0$, which implies, in view of the complementarity condition, that  $\alpha_{m_{1,k}}\alpha_{m_{2,k}}= 0$, and, after substitution in the first two equations, we have the two KKT points
$$
\left(\frac{s_k}{1+\rho},0\right),\ \ \ \left(0,-\frac{s_k}{1+\rho}\right).
$$
The objective function value at both these KKT points is equal to $\frac{\rho}{1+\rho} s_k^2$, lower than the objective function value at the origin, and, thus, these KKT points are the two global minima of the restricted problem (\ref{eq:restricted}).
\end{itemize}
Based on the above result, we have that problem
$$
\begin{array}{lll}
\text{minimize}_\vec{\alpha} & \sum_{k=1}^K (\alpha_{m_{1,k}}-\alpha_{m_{2,k}}-s_k)^2 + \\ 
& \quad\rho\sum_{k=1}^K (\alpha_{m_{1,k}})^2 +\rho \sum_{k=1}^K (\alpha_{m_{2,k}})^2  & \\
\text{s.t.} & \alpha_{m_{1,k}}\alpha_{m_{2,k}}\geq 0 \quad  \forall k \in \{1,\ldots,K\},
\end{array}
$$
which is the original one (\ref{eq:polytrans})  without the last term $\left[\sum_{k=1}^K (\alpha_{m_{1,k}}+\alpha_{m_{2,k}})\right]^2$, and which can be split into the $K$ subproblems (\ref{eq:restricted}), has global minimum value
equal to  $\frac{\rho \sum_{k=1}^K s_k^2}{1+\rho}$  and $2^K$ global minima defined as follows: for each $I_1, I_2\subseteq \{1,\ldots, K\}$ such that $I_1\cap I_2=\emptyset$ and $I_1\cup I_2=\{1,\ldots,K\}$,
$$
\alpha_{m_{1,k}}^\star=
\left\{\begin{array}{ll}
\frac{s_k}{1+\rho} & k\in I_1 \\
0 & k\not\in I_1,
\end{array}
\right.\ \ \ 
\alpha_{m_{2,k}}^\star=
\left\{\begin{array}{ll}
-\frac{s_k}{1+\rho} & k\in I_2 \\
0 & k\not\in I_2.
\end{array}
\right.
$$
Now, if we replace these coordinates in the omitted term  $\left[\sum_{k=1}^K (\alpha_{m_{1,k}}+\alpha_{m_{2,k}})\right]^2$, we have the following
$$
\left[\sum_{k=1}^K (\alpha_{m_{1,k}}^\star+\alpha_{m_{2,k}}^\star)\right]^2=\frac{1}{(1+\rho)^2}\left[\sum_{k\in I_1} s_k -\sum_{k\in I_2} s_k\right]^2,
$$
which is equal to 0 for some $I_1, I_2$  if and only if the \subsetsum problem admits a solution. As a consequence the optimal value of problem (\ref{eq:polytrans}) is equal to
$\frac{\rho \sum_{k=1}^K s_k^2}{1+\rho}$ if and only if the \subsetsum problem admits a solution, as we wanted to prove.
\end{proof}
 \section{Experiments}
\label{sec:experiments}

\definecolor{plotgreen}{RGB}{40,180,35}
\definecolor{plotorange}{RGB}{225,147,15}
\definecolor{plotblue}{RGB}{31,119,180}

In this section, we present the experimental findings obtained through the application of the algorithms proposed in this paper.

In Subsection~\ref{ssec:exp-time-vs-obj}, we investigate the properties in terms of regression performance and runtime of \algoopt, \algoalt, and \algobnb, providing insights about when one should be preferred over the other.

In Subsection~\ref{ssec:interpretability}, we provide an example of interpreting the solution provided by our approach. Unfortunately, interpretability is not easily measurable and is, in general, highly task-dependent~\cite{doshi2017towards}.
Nonetheless, previous research has discussed interpretability of models across multiple dimensions, such as simulatability, decomposability and algorithmic transparency~\citep{lipton16mythos}.
To show the benefits of framing a regression task as a partitioned least squares problem, we report an experiment analyzing the solution found by the \algoopt algorithm on an additional dataset (the Ames House Prices dataset).
In particular, we will show that the grouped solution found via the Partitioned Least Squares formulation is arguably more simulatable and decomposable compared to the more commonly employed ``feature-by-feature'' linear regression solutions.
Finally, in Subsection~\ref{ssec:model-quality}, we compare the generalization performances of our approach with those of least squares and two established variants: Partial Least Squares (PLS) and Principal Component Regression (PCR).

\pgfplotsset{yticklabel style={text width=3em,align=right}}

\begin{figure}
\centering
\begin{subfigure}{0.49\textwidth}
\begin{tikzpicture}
\begin{axis}[
    scale=0.7,
    xmode=log,
    xmin=0, xmax=10,
    grid=both,
    major grid style={black!50}
]
\addplot[orange, mark=*] table [x=TimeCumulative, y=TrainBest, col sep=comma]{Limpet-results-ALT-20.csv};
\addplot[NavyBlue, mark=*] table [x=TimeCumulative, y=TrainBest, col sep=comma]{Limpet-results-ALT-100.csv};
\addplot[OrangeRed, mark=*] table [x=TimeCumulative, y=TrainBest, col sep=comma]{Limpet-results-OPT.csv};
\addlegendentry{PartLS-Alt T20}
\addlegendentry{PartLS-Alt T100}
\addlegendentry{PartLS-Opt}
\end{axis}
\end{tikzpicture}
\caption{Limpet}
\label{fig:limpet}
\end{subfigure}
\begin{subfigure}{0.49\textwidth}

\begin{tikzpicture}
\begin{axis}[
    scale=0.7,
    xmode=log,
grid=both,
    major grid style={black!50}
]
\addplot[orange, mark=*] table [x=TimeCumulative, y=TrainBest, col sep=comma]{Supercond-results-ALT-20.csv};
\addplot[NavyBlue, mark=*] table [x=TimeCumulative, y=TrainBest, col sep=comma]{Supercond-results-ALT-100.csv};
\addplot[OrangeRed, mark=*] table [x=TimeCumulative, y=TrainBest, col sep=comma]{Supercond-results-OPT.csv};
\end{axis}
\end{tikzpicture}
    
\caption{Superconductivity}
\label{fig:second}
\end{subfigure}
\hfill
\begin{subfigure}{0.49\textwidth}
\begin{tikzpicture}
\begin{axis}[
    scale=0.7,
    xmode=log,
grid=both,
    major grid style={black!50}
]
\addplot[orange, mark=*] table [x=TimeCumulative, y=TrainBest, col sep=comma]{Facebook-results-ALT-20.csv};
\addplot[NavyBlue, mark=*] table [x=TimeCumulative, y=TrainBest, col sep=comma]{Facebook-results-ALT-100.csv};
\addplot[OrangeRed, mark=*] table [x=TimeCumulative, y=TrainBest, col sep=comma]{Facebook-results-OPT.csv};
\end{axis}
\end{tikzpicture}    \caption{Facebook Comment Volume}
    \label{fig:third}
\end{subfigure}
\begin{subfigure}{0.49\textwidth}
\begin{tikzpicture}
\begin{axis}[
    scale=0.7,
    xmode=log,
grid=both,
    major grid style={black!50}
]
\addplot[orange, mark=*] table [x=TimeCumulative, y=TrainBest, col sep=comma]{YearPrediction-results-ALT-20.csv};
\addplot[NavyBlue, mark=*] table [x=TimeCumulative, y=TrainBest, col sep=comma]{YearPrediction-results-ALT-100.csv};
\addplot[OrangeRed, mark=*] table [x=TimeCumulative, y=TrainBest, col sep=comma]{YearPrediction-results-OPT.csv};
\end{axis}
\end{tikzpicture}
    
    \caption{YearPredictionMSD}
    \label{fig:fourth}
\end{subfigure}

\caption{Plot of the behavior of the two proposed algorithms on four datasets. \algoalt has been repeated 100 times following a multi-start strategy and in two settings (\textcolor{plotblue}{$T$=20} and \textcolor{plotorange}{$T$=100}). Each point on the orange and blue lines reports the cumulative time and best objective found during these 100 restarts. \algoopt outputs a single solution, drawn in \textcolor{plotgreen}{green}.}
\label{fig:time-vs-obj-figures}

\end{figure}

\begin{table}
\centering
\caption{Summary of dataset statistics. We stress that the number of feature groups may be adapted depending on the analyst's needs and the constraints they wish to impose on the regressor.}
\label{tab:dataset-summary}
\begin{tabular}{llll}
\toprule
Dataset                 & Rows  & Columns & \# Feature Groups \\ 
\midrule
Limpet                  & 82    & 44      & 6                 \\ 
Facebook Comment Volume & 40000 & 53      & 5                 \\ 
Superconductivity       & 10000 & 81      & 7                 \\ 
YearPredictionMSD       & 10000 & 90      & 9                 \\ 
Ames House Regression   & 2931  & 79      & 10                \\ 
\botrule
\end{tabular}

\end{table}

\subsection{Runtime vs. Solution Quality}\label{ssec:exp-time-vs-obj}

We start by experimenting with \algoopt and \algoalt and  on four regression problems on the following datasets: \emph{Limpet}, \emph{Facebook Comment Volume}, \emph{Superconductivity}, and \emph{YearPredictionMSD}. Details about these datasets may be found in the Appendix.
We choose these datasets because of their relatively high number of features. In particular, the Limpet dataset had already been the subject of a block-relevance analysis in previous literature~\citep{ermondi12molecular,caron13block}.
We ran \algoalt (Algorithm~\ref{algo:pls-a})  in a multi-start fashion with 100 randomly generated starting points. 
The four panels in Figure~\ref{fig:time-vs-obj-figures} report the best objective value obtained during these random restarts along with the cumulative time needed to obtain that value (so the rightmost point will plot the cumulative time of the 100 restarts versus the best objective obtained in the whole experiment).
We repeated the experiment using two different values of parameter $T$ (number of iterations), setting it to $20$ and $100$, respectively. So for a single random restart with $T=20$ (or $T=100$), Algorithm~\ref{algo:pls-a} will alternate $20$ ($100$) times before returning. 
As one would expect, we see that increasing the value of parameter $T$ slows down the algorithm, but allows it to converge to better solutions.

The experiments confirm that \algoopt retrieves more accurate solutions, as expected due to its global optimality property established in Section \ref{ssec:pls-b}. Depending on the dataset, this solution may be either cheaper or more costly to compute compared to the approximate solution obtained by \algoalt. Notably, in typical scenarios, the alternating least squares approach, \algoalt, outperforms \algoopt in terms of running time only when the total number of iterations (and thus the total number of convex subproblems to be solved) is smaller than $2^K$. 
However, in our experimentation, this often results in solutions that grossly approximate the optimal one. Consequently, we find that \algoopt is likely preferable in most cases, providing an optimal solution within a reasonable timeframe, often even quicker than \algoalt. Furthermore, although the alternating algorithm can occasionally deliver a solution faster than \algoopt, which might be deemed "good enough", its iterative nature introduces a degree of uncertainty.

Clearly there are cases where the number of groups or where the time required to solve a single convex problem is very large. In these cases, when approximate solutions are acceptable for  the application at hand, \algoalt could be a very compelling solution. We conclude by noting that a use case with a large number of groups appears to us not very plausible. In fact, it could be argued that the reduced interpretability of the results defies one of the main motivations behind employing the Partitioned Least Squares model in the first place. 

It is worth mentioning that, in case a problem with a large $K$ were to arise, the \algobnb algorithm (see Algorithm~\ref{algo:pls-bnb})  is likely to allow users to retrieve the optimal solution more efficiently than \algoopt. We propose here a further experiment with synthetic data, through which we show when it is convenient to switch from \algoopt to the Branch-and-Bound approach implemented in \algobnb and discussed in Section \ref{sec:branch-and-bound}.
In all the previously discussed experiments the cardinality $K$ of the partition is relatively small. Thus, \algoopt is able to solve the related problems efficiently. 
However, the computing times of \algoopt quickly increases exponentially as $K$ increases. In these cases, a Branch-and-Bound approach is a much better choice. To better clarify this point, we report in Table \ref{tab:diffinst} the results on synthetic data obtained by randomly generating in the interval $[-10,10]$ the entries of $\mat{X}$, by generating $\vec{y}$ by adding some random noise generated in the interval $[-50,50]$ to each entry of a target solution
$\vec{y}_{\text{ref}}=\mat{X}\vec{w}_{\text{ref}}$, and by randomly generating the $K$ sets in the partition. In the table we compare the computing times (in seconds), for different values $K,N,M$, of \algoopt and of Algorithm~\ref{algo:pls-bnb} (with lower bounds at branch-and-bound nodes computed through the solution of least squares problems with additional non-negativity constraints). A $-$ denotes a computing time exceeding 1,000 seconds. The results clearly show that, as $K$ increases,  a branch-and-bound approach is much more efficient than \algoopt.
\begin{table}
\caption{Computing times (in seconds) for \algoopt and Algorithm ~\ref{algo:pls-bnb}, for different values of $K,N,M$. \label{tab:diffinst}}
\centering
\begin{tabular}{ccccc}
\toprule
$K$ & $N$ & $M$  & Time \algoopt & Time BB \\
\midrule
10 & 100 & 400 &  3.62 & 0.81 \\ 
10 & 100 & 600 &  3.74 & 0.58  \\ 
10 & 150 & 400 & 9.50 & 1.15  \\ 
10 & 150 & 600 & 9.63 & 2.52  \\ 
10 & 200 & 400 &  14.90 & 2.09  \\ 
10 & 200 & 600 &  15.19 & 7.05  \\ 
15 & 100 & 400 &  141.71 & 1.70  \\ 
15 & 100 & 600 &  151.29 & 1.65  \\ 
15 & 150 & 400 &  346.65 & 5.95  \\ 
15 & 150 & 600 &  403.24 & 4.31  \\ 
15 & 200 & 400 &  525.27 & 8.24  \\ 
15 & 200 & 600 &  595.25 & 18.26  \\ 
20 & 100 & 400 &  - & 1.71  \\ 
20 & 100 & 600 &  - & 1.42  \\ 
20 & 150 & 400 &  - & 13.26  \\ 
20 & 150 & 600 &  - & 3.67  \\ 
20 & 200 & 400 &  - & 22.57  \\ 
20 & 200 & 600 &  - & 6.60  \\
\botrule
\end{tabular}
\end{table}

\subsection{Interpretability on Ames House Prices}
\label{ssec:interpretability}

We present here an analysis of a solution found by \algoopt on the Ames House Prices dataset, which is publicly available via Kaggle~\cite{kaggle}.
This dataset has a relatively high number of columns --- 79 in total --- each detailing one particular characteristic of housing properties in Ames, Iowa. 
The task is to predict the selling price of each house. 

We propose a grouping of the features into 10 groups, each one representing a high-level characteristic of the property (see Table~\ref{tab:ames-groups}).
As an example, we collect 6 columns referring to the availability and quality of air conditioning systems, electrical system, heating and fireplaces in a ``Power and Temperature'' group. 
Other feature groups refer to overall quality of the construction work and materials employed (``Building Quality''), external facilities such as garages or swimming pools (``Outside Facilities''). 
We show the feature groups we designed and the $\beta$ values found by  \algoopt \footnote{In this the regularization parameter has been set to $\rho=10$} in Figure~\ref{fig:ames-house-beta}. 
We note that the grouped solution enabled by the partitioned least squares formulation is able to give a high-level summary of the regression result. 
An analyst is therefore able to communicate easily to, e.g. an individual selling their house, that the price is mostly determined by the building quality and the attractiveness of the lot. 
A deeper analysis is of course possible by investigating the $\alpha$ values found by the algorithm.
For instance, we report the $\alpha$s contributions for the ``Outside Facilities'' group in figure~\ref{fig:ames-house-alpha}. 
Here, one is able to notice that garage quality has the biggest impact on the property's price, which is potentially actionable knowledge. 
\begin{figure}[t]
	\centering
	\includegraphics[width=\columnwidth]{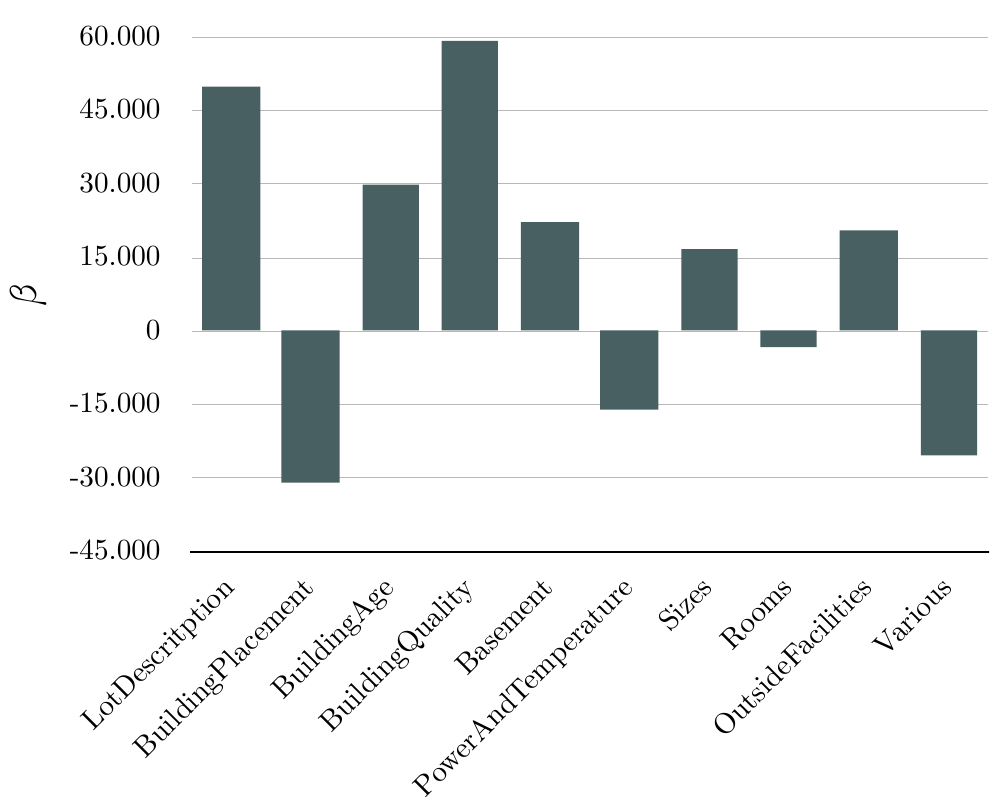}
	\caption{\label{fig:ames-house-beta}Feature groups and associated $\beta$ values as found by \algoopt on the Ames House Prices dataset~\cite{kaggle}.}
\end{figure}

\begin{figure}[t]
	\centering
	\includegraphics[width=\columnwidth]{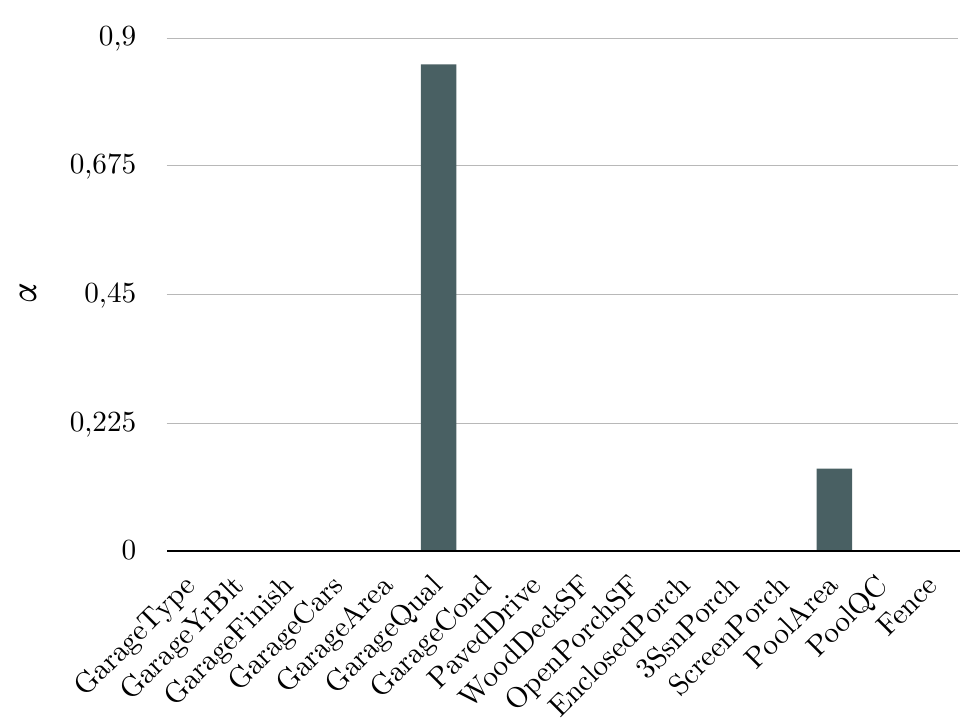}
	\caption{\label{fig:ames-house-alpha}$\alpha$ values for group ``Outside Facilities'' found by \algoopt on the Ames House Prices dataset~\cite{kaggle}.}
\end{figure}
In Figure~\ref{fig:lsq-ames-house-outside}, we report the weights of the features in the ``Outside Facilities'' group as learnt by the least squares algorithm.
\begin{figure}
\includegraphics[width=\columnwidth]{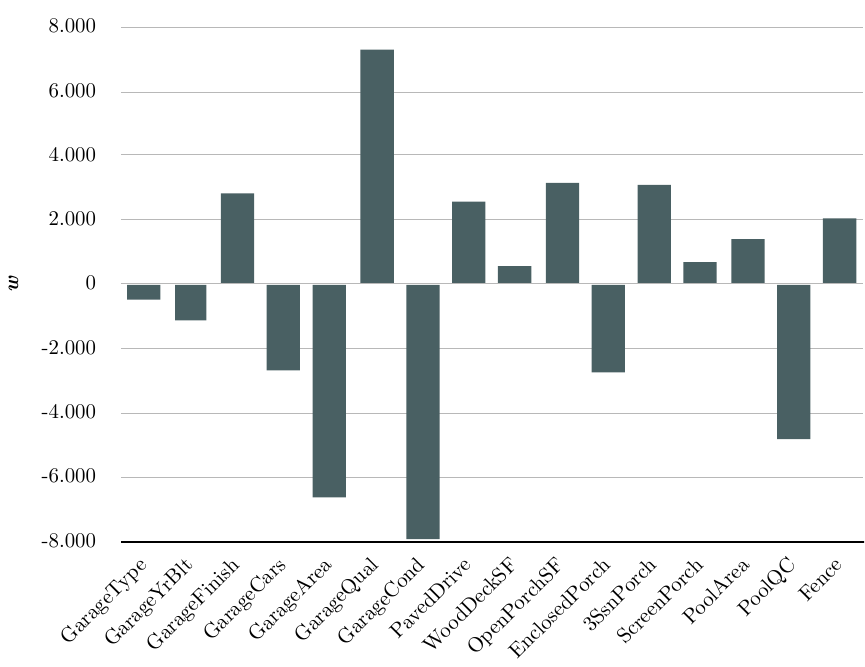}
\caption{\label{fig:lsq-ames-house-outside}Feature weights for group ``Outside Facilities'' of a regularized linear regression on the Ames House Prices dataset~\cite{kaggle}.}
\end{figure}
We argue that the group- and feature-level analysis made possible by our contributions improves on the interpretability of ungrouped linear regression.
While linear regression is a relatively simple model and therefore intuitively satisfies some notion of transparency, previous work has established that this is not necessarily the case.
Lipton~\cite{lipton16mythos} discusses interpretability of models around three separate dimensions: simulatability, decomposability and algorithmic transparency. Simulatability may be achieved when a person is able to contemplate all the model at once in reasonable amount of time. 
While the amount of time is of course subjective, Lipton stresses the fact that linear models may not be simulatable if a high number of features are involved. 
On the other hand, the partitioned least squares formulation we propose finds a higher-level, grouped solution via the $\beta$ values. Thus, a practitioner would be able to build a simpler mental model of the solution by focusing on the groups rather than the individual features.

\subsection{Quality of the Inferred Model}\label{ssec:model-quality}

\newcommand{\stdf}[1]{\ \mbox{\tiny $\pm$ #1}}

\begin{table}[p]

\caption{Train/Test mean squared errors of Least Squares (LS), Principal Component Regression (PCR), Partial Least Squares (PLS) and Partitioned Least Squares (PartLS). Partitioned Least Squares has been run with several configurations of the partition matrix: in ``P from DK'' experiments the partition matrix is built from domain knowledge; in ``P arbitrary'' the partition matrix is arbitrary (used instead of DK when domain knowledge was not available); in ``P from LS'' we use a binary partition matrix based on the signs of the attributes in the solution of the LS problem computed on the training set; in ``P opt'' the partition matrix is built as in ``P from LS'', but the LS is run over the complete dataset. We report the average and the standard deviation of results over 100 repetitions based on different train/test splits. Bold faced averages are statistically better at the 99\% significance level (using a dpaired t-test). \label{tab:msq-errors}}

\begin{tabular}{llr@{}p{4em}r@{}p{4em}}
\toprule
Dataset & Method & \multicolumn{2}{c}{Train} & \multicolumn{2}{c}{Test} \\
\midrule
\multirow{5}{*}{Artificial} & LS & 0.1759 & \stdf{0.04} & 0.8060 & \stdf{0.16} \\
 & PCR & 23.5646 & \stdf{4.48} & 30.3264 & \stdf{2.26}\\
 & PLS & 0.5060 & \stdf{0.10} & 2.0636 & \stdf{0.64}\\
 & PartLS: P from LS & 0.1759 & \stdf{0.04} & 0.8060 & \stdf{0.16} \\
 & PartLS: P from DK & 0.1885 & \stdf{0.05} & \textbf{0.6845} & \stdf{0.11} \\
 & PartLS: P opt & 0.1885 & \stdf{0.05} & \textbf{0.6845} & \stdf{0.11} \\[1em]
\multirow{5}{*}{Limpet} & LS & 0.0000 & \stdf{0.00} & 4.1391 & \stdf{3.95} \\
 & PCR & 0.7085 & \stdf{0.13} & \textbf{1.0548} & \stdf{0.39}\\
 & PLS & 0.6364 & \stdf{0.76} & \textbf{1.2695} & \stdf{1.65} \\
 & PartLS: P from LS & 0.0000 & \stdf{0.00} & 6.9099 & \stdf{12.56} \\
 & PartLS: P from DK & 0.0000 & \stdf{0.00} & 46.3429 & \stdf{47.30} \\
 & PartLS: P opt & 0.0000 & \stdf{0.00} & 4.4367 & \stdf{4.20} \\[1em]

\multirow{5}{*}{Facebook} & LS & 2785.6553 & \stdf{450.79} & 3012.9160 & \stdf{725.57} \\
 & PCR & 3672.9801 & \stdf{503.71} & 3687.1780 & \stdf{789.64} \\
 & PLS & 3765.2822 & \stdf{509.81} & 3784.4111 & \stdf{798.20} \\
 & PartLS: P from LS & 2785.6553 & \stdf{450.79} & 3012.9160 & \stdf{725.57} \\
 & PartLS: P arbitrary & 2850.4851 & \stdf{457.82} & 3023.1348 & \stdf{733.19} \\
  & PartLS: P opt & 2791.9103 & \stdf{450.93} & \textbf{2988.1926} & \stdf{722.48} \\[1em]

\multirow{5}{*}{Year Prediction} & LS & 89.9616 & \stdf{1.97} & 91.8708 & \stdf{2.20} \\
& PCR & 115.2383 & \stdf{2.39} & 115.5236 & \stdf{2.71} \\
& PLS & 3.42E5 & \stdf{6.8E3} & 3.5E5 & \stdf{1.1E4} \\
& PartLS: P from LS & 89.9616 & \stdf{1.97} & 91.8708 & \stdf{2.20} \\
& PartLS: P arbitrary & 100.2951 & \stdf{2.02} & 101.8315 & \stdf{2.28} \\
& PartLS: P opt & 90.1798 & \stdf{1.97} & \textbf{91.7042} & \stdf{2.18} \\[1em]

\multirow{5}{*}{Superconductivity} & LS & 307.6365 & \stdf{3.68} & 310.8739 & \stdf{5.76} \\
& PCR & 762.7407 & \stdf{6.77} & 764.4683 & \stdf{11.35} \\
& PLS & 566.8627 & \stdf{5.25} & 567.6426 & \stdf{8.32} \\
& PartLS: P from LS & 307.6365 & \stdf{3.68} & 310.8739 & \stdf{5.76} \\
& PartLS: P arbitrary & 376.8152 & \stdf{3.91} & 379.3639 & \stdf{5.87} \\
& PartLS: P opt & 307.7331 & \stdf{3.68} & \textbf{310.7671} & \stdf{5.76} \\
\botrule
\end{tabular}

\end{table}

While one of the major benefits of the Partitioned Least Squares problem is in simplifying the interpretation of the results, it should be self-evident that this would be a pointless exercise if the returned hypothesis were not at least comparable to other widely used techniques in terms of generalization capabilities. In this section, we investigate generalization performances of regressors learnt by \algoopt and compare them with Least Squares (LS), Principal Component Regression (PCR) and Partial Least Squares (PLS). All experiments are repeated 100 times on different train/test splits. 

We experiment on the four datasets earlier in this section and on an additional dataset \textit{Artificial} which we created for this specific test. The main goal of this dataset is to showcase a situation where we have complete and accurate domain knowledge about the partition. The artificial dataset contains $70$ training samples and $930$ test samples. Samples contains feature values randomly sampled from a normal distribution. This dataset's target variable may be computed without cross-partition feature interactions: Specifically, the target is computed as $\mathbf{y} = \mathbf{X} \times (\mathbf{P} \circ \boldsymbol{\alpha}) \times \boldsymbol{\beta} + t$, where $\mathbf{P}$ is a partition in 5 sets having cardinalities $5,10,4,12,6$. The $\mathbf{X}$ matrix has been perturbed with gaussian noise with mean $0$ and standard deviation $0.05$ after generating the target column.
$t$ and $\boldsymbol{\alpha}$ have been generated using a uniform distribution in $[0,1]$. $\boldsymbol{\alpha}$ has then been normalized so that 
$P^\top \times \boldsymbol{\alpha} = \mathbf{1}$. The $\boldsymbol{\beta}$ are the normalization factors used to ensure $P^\top \times \boldsymbol{\alpha} = \mathbf{1}$ multiplied by $11, 4, 2, 1$, and $3$. Signs of the groups have been set to $-1, 1, 1, -1,$ and $1$. These latter parameters have been set arbitrarily and without tuning. 

Results are reported in Table~\ref{tab:msq-errors}. When the test error of a method is significantly better\footnote{According to a paired t-test at the 99\% confidence level.} than the competitors, it is shown in bold. If more than one result is in bold, then the bold-faced results are not significantly better with respect to each other, but are significantly better than all the remaining results. 

The PCR algorithm has been run setting the maximal number of principal components equal to the number of groups in the $P$ matrix. We conducted experiments with \algoopt three times, utilizing three distinct partitioning methods: one based on the partitions devised by ourselves (``P arbitrary'' or ``P from DK''), one where features were grouped based on their signs in the solution found by LS in the experiment with the same train/test split (``P from LS''), and one (``P opt'') using this same methodology but on the results of LS run on the full dataset (i.e., before the train/test split). This latter experiment aims at simulating a situation where the domain knowledge closely matches the ``natural'' partitioning of the columns of the dataset. 
The settings ``P from LS'' and ``P opt'' aim to demonstrate that, with the correct partitioning, the algorithm converges to an optimal solution. 

We start discussing these results by noting that, for all datasets except Limpet where the problem is heavily under-determined, LS and \algoopt on ``P from LS'' yield identical results. This is expected as the problems are equivalent from the optimizer perspective. Furthermore, results of \algoopt on ``P opt'' show that, when the provided partition is accurate, the inductive bias allows for better generalization in most situations.

On the Artificial dataset, \algoopt on ``P from DK'' significantly outperforms the competitors. This shows once more that, when the correct partitioning is provided, \algoopt exhibits an inductive bias that enhances generalization. On this dataset \algoopt attains the same results when using ``P from DK'' and ``P opt''. This is expected since ``P from DK'' is built using perfect knowledge of the $P$ matrix, and we verified that the signs of the features found by LS on the complete dataset induce a partition that can be formally shown to be equivalent to the one that have been used to generate the data. 

The PCR and the PLS algorithms are the clear winners on the Limpet dataset. The dataset matrix is under-determined and collinear, which is the ideal case for these techniques. In all the other cases, their inductive biases significantly hinder the algorithms performances, most likely because the number of principal components guessed on the basis of the $P$ matrix was not sufficient to explain enough variance in the data.  Setting the number of maximal number of principal components to be equal to the number of features does not seem to change much the results: either they converge to the LS solution, or they obtain result not too distant from the ones presented in Table~\ref{tab:msq-errors}.

For Facebook, Year Prediction, and Superconductivity datasets, \algoopt yields the best performances when equipped with the ``P opt'' partition. It lags a little behind LS when equipped with the partitions we used in our previous experiments (``P arbitrary''), which is totally reasonable since those partitions were chosen to showcase the difference in the time performances between the approaches, rather than the quality of the generalization results.

The experiments overall demonstrate that the constraints imposed by the Partitioned Least Squares approach can serve as a strong inductive bias when the partition knowledge is accurate. However, the technique encounters difficulties when analyzing datasets with many collinear features. Indeed, the current formulation of Partitioned Least Squares does not address this specific issue, suggesting that further research is needed to tackle this challenge.
 \section{Conclusions}
\label{sec:conclusions}

In this paper we presented an alternative least squares linear regression formulation. Our model enables scientists and practitioners to group features together into partitions, hence allowing the modeling of higher level abstractions which are easier to reason about. We provided rigorous proofs of the non-convexity of the problem and presented \algoalt and \algoopt, two algorithms to cope with the problem.

\algoalt is an iterative algorithm based on the alternating least squares method. The algorithm is proved to converge, but there is no guarantee that the accumulation point results in a globally optimal solution. On the contrary, as experiments have shown, the algorithm can be trapped in a local minimizer and return an approximate solution. Experiments suggest that it could be faster and preferable to the exact algorithm \algoopt in some circumstances (e.g., when the time needed to solve a single sub-problem is large and the application allows for sub-optimal answers).

\algoopt is an enumerative, exact, algorithm and our contribution includes a formal optimality proof. In our experimentation, we confirmed that it behaves very well under several different settings, although its time complexity grows exponentially with the number of groups. We argue that this exponential growth in time complexity should not impede its adoption: a large number of groups seems implausible in practical scenarios since it would undermine interpretability of the results and hence the attractiveness of the problem formulation. However, for the sake of completeness and to provide guidance to the interested reader, we provided a branch-and-bound solution that shares the same optimality guarantees of \algoopt. This latter formulation, depending on the actual structure of the problem as implied by the data, might save computation by pruning the search space, possibly avoiding to solve a large number of sub-problems. In Section \ref{ssec:exp-time-vs-obj} we have shown the benefits of this strategy when the number of partition sets increases, but we intend to further investigate this issue in future work.

In Section~\ref{ssec:model-quality}, we explore how the constraints introduced by the Partitioned Least Squares formulation impact the generalization properties of the inferred model. Our findings indicate that when the partition knowledge aligns with the underlying data distribution, the Partitioned Least Squares algorithms are very effective in leveraging this information. However, the results obtained from the Limpet dataset clearly demonstrate that collinearity can pose a challenge for the proposed technique and, indeed, neither the problem formulation, nor the proposed algorithms try to address this issue. We believe that addressing collinearity problems represents an interesting avenue for future research.

One topic for further research is about how to evaluate the partitions created by a domain expert. 
In this work, we have taken feature partitions ``at face value'' or otherwise assumed that an agreed-upon partitioning was developed by an expert.
Investigating the challenges of the (human) partitioning process, possibly by performing an interactive user study as suggested by Doshi-Velez et al. \cite{doshi2017towards}, is a possible avenue for future developments.

\section*{Funding}

None

\section*{Conflicts of interest/Competing interests}

None

\section*{Ethics approval}

Not applicable

\section*{Consent to participate}

Not applicable

\section*{Consent for publication}

Not applicable

\section*{Code availability}

A Julia~\citep{bezanson12fast} implementation of algorithms \algoopt, \algoalt, and \algobnb is available at~\url{https://github.com/ml-unito/PartitionedLS}; the code for the experiments is available at:~\url{https://github.com/ml-unito/PartitionedLS-experiments-2}.\\

\section*{Availability of data and material}

All datasets, with the exception of the LIMPET dataset, are publicly available. The LIMPET dataset can be obtained by contacting the authors of the original paper~\citep{caron13block}.

The repository for the experiments contains code to download and to pre-process the datasets (or the datasets themselves when not available for downloading) as well as the scripts to actually launch the experiments. Pre-processing consists in packing the data in a format suitable for the algorithms and to partition the data as mentioned in Section~\ref{sec:experiments}.

\section*{Authors' contributions}

Roberto Esposito, Marco Locatelli, and Mattia Cerrato all contributed to the conception and design of the work. Roberto Esposito and Marco Locatelli designed the original version of the algorithms. Roberto Esposito, Marco Locatelli and Mattia Cerrato designed the experiments and contributed to the writing of the manuscript. All authors read and approved the final manuscript.

\vskip 0.2in
\bibliography{biblio}

\newpage

\section*{Appendix}

\subsection*{Dataset Descriptions}
\label{sec:datasets}
\subsubsection*{Limpet dataset}
This dataset~\citep{caron16fast} contains 82 features describing measurements over simulated (VolSurf+ ~\citep{goodford85computational}) models of 44 drugs. The regression task is the prediction of the lipophilicity of the 44 compounds. The 82 features are partitioned into 6 groups according to the kind of property they describe. The six groups have been identified by domain experts and are characterized in~\citep{ermondi12molecular} as follows:
\begin{itemize}
    \item \textbf{Size/Shape}: 7 features describing the size and shape of the solute; 
    \item \textbf{OH2}: 19 features expressing the solute's interaction with water molecules;
    \item \textbf{N1}: 5 features describing the solute's ability to form hydrogen bond interactions with the donor group of the probe; 
    \item \textbf{O}: 5 features expressing the solute's ability to form hydrogen bond interactions with the acceptor group of the probe;
    \item \textbf{DRY}: 28 features describing the solute's propensity to participate in hydrophobic interactions;
    \item \textbf{Others}: 18 descriptors describing mainly the imbalance between hydrophilic and hydrophobic regions.
\end{itemize}
This dataset, while not high-dimensional in the broadest sense of the term, can be partitioned into well-defined, interpretable groups of variables. Moreover and perhaps more importantly, this is a clear case where the Partitioned Least Squares formulation is important to correctly handle the structure of the problem: each group contains variables describing phisical properties of the compound that are theoretically bound to act in the same ``direction'' on the target variable (its lipophilicity). Previous literature which employed this dataset has indeed focused on leveraging the data's structure to obtain explainable results~\citep{caron13block}. 
We used as training/test split the same one proposed in \citep{caron16fast}.

For this particular problem, the number of groups is $6$ and \algoopt needs to solve just $2^6=64$ convex problems. It terminates in $\sim 1.4$ seconds reaching a value of the objective function of about $4.3 \cdot 10^{-14}$ (note that the annotation ``$1e-13$'' at the top of the plot denotes that all values on the $y$ axis are to be multiplied by $10^{-13}$). \algoalt (Algorithm~\ref{algo:pls-a}) in this particular case is doing very well. Even though the plot shows that \algoopt reaches a better loss value, \algoalt starts already at a very low value of about $3 \cdot 10^{-13}$ requiring a fraction of the time needed by its optimal counterpart. It is also worth noting that, despite the small changes in the objective value reached by the two algorithms, the configuration of the $\vec{\alpha}$ and $\vec{\beta}$ variables are substantially different.

\subsubsection*{Facebook Comment Volume Dataset}

The Facebook Comment Volume dataset \citep{Sing1601:Facebook} contains more than 40 thousand training vectors along with 53 features. Each sample represents a post published on the social media service by a ``Facebook Page'', an entity which other users can follow and ``like'' so to receive updates on their Facebook activity. Features range from the number of users which ``like'' and follow the page to the number of comments the post received during different time frames. We removed the column which indicated whether a post was a paid advertisement, as this feature only contained 0 values, i.e., no advertisements were collected. Then, we divided the features into 5 blocks, each containing 10 features save for the last one which contained 11 features. The task here is to predict how many comments the same post will receive in the next few hours. The dataset is hosted at the UCI repository \citep{dua2019uci}. To keep training time and memory usage low, we limited the training samples to the first $15000$ examples of the training set. On this dataset, \algoopt is able to find the highest quality solution in less than 5 seconds. \algoalt with $T=20$ finds a similar quality solution after about 7 seconds.  \algoalt with $T=100$ takes more than 3 minutes to converge to a comparable objective value. 

\subsubsection*{Superconductivity dataset}

The Superconductivity dataset contains 81 features representing characteristics of superconductors. The dataset contains $21264$ examples. In our experiment we trained the model over the first $10000$ examples. The task is to predict a material's critical temperature. The features are derived from a superconductor's atomic mass, density and fusion heat among others. We refer the reader to the original paper~\citep{hamidieh2018data} for the specific details about the process.
In our experiment, we created 7 feature blocks with 10 features each and an additional one which contained 11 features. \algoopt takes $\sim 47$ seconds reaching an objective value of $\sim 2051$. At about the same computational cost,  \algoalt with $T=20$ reaches an objective of $\sim 2150$. It will take the algorithm about $\sim 440$ seconds to lower that figure to a loss objective value ($\sim 2072$) comparable to the one obtained by  \algoopt. Setting $T=100$ slightly improves the situation: after about $40$ seconds the loss objective is $\sim 2117$, which lowers to $\sim 2080$ after $\sim 186$ seconds and to $\sim 2064$ after $\sim 881$ seconds.

\subsubsection*{YearPredictionMSD Dataset}

We also propose an experimentation on the YearPredictionMSD dataset. It is a subset of the Million Songs dataset~\citep{Bertin-Mahieux2011}. When compared with the original dataset, it has about half the examples (around 500 thousands) and instead of the raw audio and metadata 90 timbre-related features are included. As for the Superconductivity dataset, we limited our experimentation to the first $10000$ examples. The target variable represents the year a song has been released in. In this dataset we experimented with 9 blocks of 10 to 12 features. \algoopt takes  $\sim 130$ seconds to reach the optimal loss at $\sim 920$.  \algoalt with $T=20$ is instead able to find a solution which is reasonably close ($\sim 922$) to the optimal one in a much shorter time (around 20 seconds). When $T=100$ is used instead,  \algoalt reaches a reasonable approximation only after $\sim 178$ seconds.

\begin{table}[h]
    \caption{Summary of the groups of features used in the Ames House Prices experiment. See the Kaggle~\cite{kaggle} repository for detailed information about the meaning of each feature label.\label{tab:ames-groups}}
    \begin{tabular}{p{0.3\columnwidth}p{0.6\columnwidth}}
    \toprule
    Group & Features \\
    \midrule
        LotDescritption & MSSubClass, MSZoning, LotFrontage, LotArea, Street, Alley, LotShape, LandContour, LotConfig, LandSlope \\[1em]
BuildingPlacement & Utilities, Neighborhood, Condition1, Condition2 \\[1em]
BuildingAge & YearBuilt, YearRemodAdd \\[1em]
BuildingQuality & BldgType, HouseStyle, OverallQual, OverallCond, RoofStyle, RoofMatl, Exterior1st, Exterior2nd, MasVnrType, MasVnrArea, ExterQual, ExterCond, Foundation, Functional \\[1em]
Basement & BsmtQual, BsmtCond, BsmtExposure, BsmtFinType1, BsmtFinSF1, BsmtFinType2, BsmtFinSF2, BsmtUnfSF, TotalBsmtSF \\[1em]
PowerAndTemperature & Heating, HeatingQC, CentralAir, Electrical, Fireplaces, FireplaceQu \\[1em]
Sizes & 1stFlrSF, 2ndFlrSF, LowQualFinSF, GrLivArea \\[1em]
Rooms & BsmtFullBath, BsmtHalfBath, FullBath, HalfBath, BedroomAbvGr, KitchenAbvGr, KitchenQual, TotRmsAbvGrd \\[1em]
OutsideFacilities & GarageType, GarageYrBlt, GarageFinish, GarageCars, GarageArea, GarageQual, GarageCond, PavedDrive, WoodDeckSF, OpenPorchSF, EnclosedPorch, 3SsnPorch, ScreenPorch, PoolArea, PoolQC, Fence \\[1em]
Various & MiscFeature, MiscVal, MoSold, YrSold, SaleType, SaleCondition \\
    \botrule
    \end{tabular}
\end{table}

\end{document}